\documentclass{article}

\usepackage{microtype}
\usepackage{graphicx}
\usepackage{subcaption}
\usepackage{booktabs} 

\usepackage{hyperref}

\usepackage[preprint]{icml2026}

\usepackage{amsmath}
\usepackage{amssymb}
\usepackage{mathtools}
\usepackage{amsthm}

\usepackage[capitalize,noabbrev]{cleveref}

\theoremstyle{plain}
\newtheorem{theorem}{Theorem}[section]

\theoremstyle{definition}
\newtheorem{definition}[theorem]{Definition}

\theoremstyle{remark}

\theoremstyle{definition}
\newtheorem{example}{Example}

\usepackage{enumitem}
\usepackage{mathrsfs}
\usepackage[on]{editing}
\usepackage{xcolor}
\usepackage{wrapfig}
\usepackage{multirow}
\usepackage{booktabs} 
\usepackage{graphicx} 
\usepackage{adjustbox} 
\usepackage{caption}   
\usepackage{makecell} 

\newcommand{\independent}{\perp\mkern-9.5mu\perp}

\newcommand{\Inv}[3]{P_{#3}\left ( #1 ; #2 \right)}

\Crefname{equation}{Eq.}{Eqs.}
\Crefname{align}{Eq.}{Eqs.}
\Crefname{figure}{Fig.}{Figs.}
\Crefname{tabular}{Tab.}{Tabs.}
\Crefname{theorem}{Thm.}{Thms.}
\Crefname{lemma}{Lem.}{Lems.}
\Crefname{proposition}{Prop.}{Props.}
\Crefname{definition}{Def.}{Defs.}
\Crefname{algorithm}{Algo.}{Algs.}
\Crefname{corollary}{Corol.}{Corol.}
\Crefname{section}{Sec.}{Sec.}
\Crefname{appendix}{App.}{Apps.}
\Crefname{prop}{Prop.}{Props.}
\Crefname{task}{Task.}{Tasks.}
\Crefname{setting}{Setting.}{Settings.}
\Crefname{example}{Example}{Expamples}

\newcommand{\ci}{\independent}

\newcommand{\Brackets}[1]{\left[ #1\right]}

\newcommand{\Parens}[1]{\left(#1\right)}

\newcommand{\angles}[2][]{#1\langle#2 #1\rangle}

\newcommand{\I}{\boldsymbol{1}}

\newcommand{\G}{\mathcal{G}}
\newcommand{\M}{\mathcal{M}}
\newcommand{\Pa}{\mli{Pa}}
\newcommand{\PA}{\mli{PA}}
\newcommand{\pa}{\mli{pa}}
\newcommand{\De}{\mli{De}}

\newcommand{\An}{\mli{An}}
\newcommand{\an}{\mli{an}}
\newcommand{\de}{\mli{de}}

\newcommand{\Ch}{\mli{Ch}}
\newcommand{\ch}{\mli{ch}}
\newcommand{\doo}{\text{do}}

\newcommand{\tuple}[2][]{\angles[#1]{#2}}

\newcommand{\mli}[1]{\mathit{#1}}

\def\*#1{\boldsymbol{#1}}
\def\1#1{\mathcal{#1}}
\def\2#1{\mathscr{#1}}
\def\3#1{\mathbb{#1}}

\newcommand{\inv}[2]{P_{#2}\left ( #1\right)}

\usepackage{pgfplots}
\DeclareUnicodeCharacter{2212}{−}
\usepgfplotslibrary{groupplots,dateplot}
\usetikzlibrary{patterns,shapes.arrows}
\pgfplotsset{compat=newest}

\usetikzlibrary{positioning,arrows,shapes,shapes.arrows,arrows.meta,trees,shapes.misc,shapes.geometric,decorations.pathreplacing,automata}
\tikzset{%
  >={Latex[width=1.5mm,length=2mm]},
  vertex/.style={draw,circle,inner sep=0mm,semithick,minimum width=4mm},
  point/.style = {circle, draw, inner sep=0.04cm,fill,node contents={}},
  uvertex/.style={outer sep=0},
  bidir/.style={<->,dashed, line width=0.25mm},
  dir/.style={->, line width=0.25mm},
  regime/.style={shape=rectangle,fill=black,inner sep=0pt,minimum size=3pt,draw},
  action/.style={shape=circle,fill=black,inner sep=0pt,minimum size=8pt,draw},
  node distance=1cm,
  font=\scriptsize\sffamily
}

\definecolor{betterred}{RGB}{228,26,28}
\definecolor{betterblue}{RGB}{55,126,184}
\definecolor{bettergreen}{RGB}{77,175,74}
\definecolor{betterpurple}{RGB}{152,78,163}
\definecolor{LightCyan}{rgb}{0.88,1,1}

\pgfdeclarelayer{back}
\pgfsetlayers{back,main}

\makeatletter
\pgfkeys{%
  /tikz/on layer/.code={
    \def\tikz@path@do@at@end{\endpgfonlayer\endgroup\tikz@path@do@at@end}%
    \pgfonlayer{#1}\begingroup%
  }%
}
\makeatother

\makeatletter
\newcommand{\xdashleftrightarrow}[2][]{\ext@arrow 3359\leftrightarrowfill@@{#1}{#2}}
\def\rightarrowfill@@{\arrowfill@@\relax\relbar\rightarrow}
\def\leftarrowfill@@{\arrowfill@@\leftarrow\relbar\relax}
\def\leftrightarrowfill@@{\arrowfill@@\leftarrow\relbar\rightarrow}
\def\arrowfill@@#1#2#3#4{%
  $\m@th\thickmuskip0mu\medmuskip\thickmuskip\thinmuskip\thickmuskip
   \relax#4#1
   \xleaders\hbox{$#4#2$}\hfill
   #3$%
}
\makeatother

\usepackage{titletoc}

\icmltitlerunning{Causal Flow Q-Learning}

\begin{document}

\twocolumn[
  \icmltitle{Causal Flow Q-Learning for Robust Offline Reinforcement Learning}



  \icmlsetsymbol{equal}{*}

  \begin{icmlauthorlist}
    \icmlauthor{Mingxuan Li}{lab1}
    \icmlauthor{Junzhe Zhang}{lab2}
    \icmlauthor{Elias Bareinboim}{lab1}
  \end{icmlauthorlist}

  \icmlaffiliation{lab1}{Causal AI Lab, Columbia University, NY, United States}
  \icmlaffiliation{lab2}{Department of Electrical Engineering and Computer Science, Syracuse University, NY, United States}

  \icmlcorrespondingauthor{Mingxuan Li}{ml@cs.columbia.edu}

  \icmlkeywords{Machine Learning, ICML}

  \vskip 0.3in
]



\printAffiliationsAndNotice{}  

\begin{abstract}
  Expressive policies based on flow-matching have been successfully applied in reinforcement learning (RL) more recently due to their ability to model complex action distributions from offline data. These algorithms build on standard policy gradients, which assume that there is no unmeasured confounding in the data. However, this condition does not necessarily hold for pixel-based demonstrations when a mismatch exists between the demonstrator's and the learner's sensory capabilities, leading to implicit confounding biases in offline data. We address the challenge by investigating the problem of confounded observations in offline RL from a causal perspective. We develop a novel causal offline RL objective that optimizes policies' worst-case performance that may arise due to confounding biases. Based on this new objective, we introduce a practical implementation that learns expressive flow-matching policies from confounded demonstrations, employing a deep discriminator to assess the discrepancy between the target policy and the nominal behavioral policy. Experiments across 25 pixel-based tasks demonstrate that our proposed confounding-robust augmentation procedure achieves a success rate 120\% that of confounding-unaware, state-of-the-art offline RL methods.
\end{abstract}

\section{Introduction}
Offline reinforcement learning (RL) offers an alternative paradigm to traditional online RL, enabling effective policy learning for decision-making from previously collected observational data when active exploration in the underlying environment is costly, unsafe, or even impractical \citep{lange2012batch,levine2020offline}. The standard offline RL problem can be interpreted as a constrained optimization: the agent seeks to maximize long-term rewards while remaining close to the state-action distributions of the observed trajectories \citep{levine2020offline}. More recently, there is a growing body of offline RL that attempts to represent more complex, multimodal policy distributions using an expressive policy class that explicitly learns the policy's velocity field via deep generative modeling \citep{mandlekar2022matters,lipman2024flow}. These expressive policy classes include denoising diffusion \citep{wang2023diffusion} and flow-matching \citep{park2025fql}. They enable modeling more complex policy distributions in the dataset, thereby enforcing accurate behavioral constraints which is critical to offline RL algorithms \citep{tarasov2023rebrac}.

\begin{figure*}[t]
    \centering
    \begin{subfigure}[c]{0.25\linewidth}
        \centering
        \includegraphics[width=\linewidth]{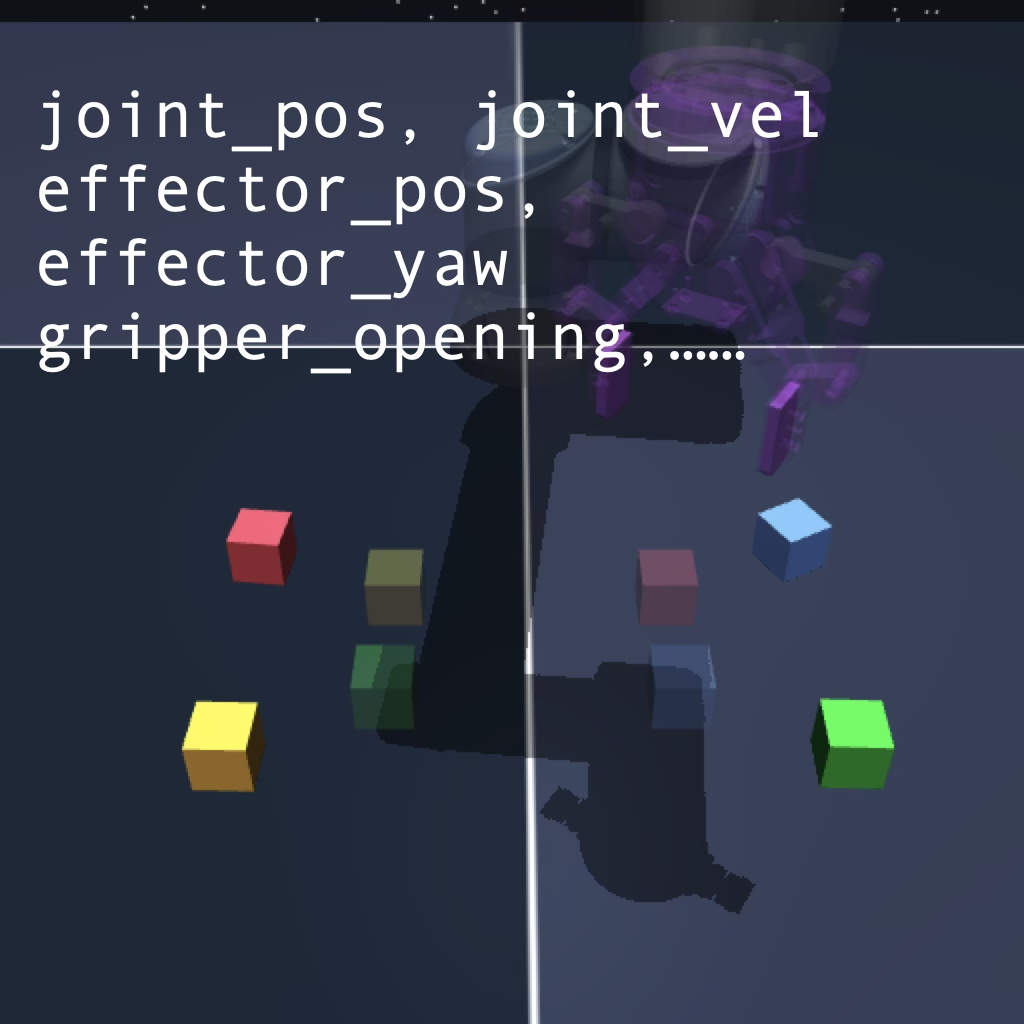}
        \caption{State Observations.}
        \label{fig:state}
    \end{subfigure}
    \hfill
    \begin{subfigure}[c]{0.25\linewidth}
        \centering
        \includegraphics[width=\linewidth]{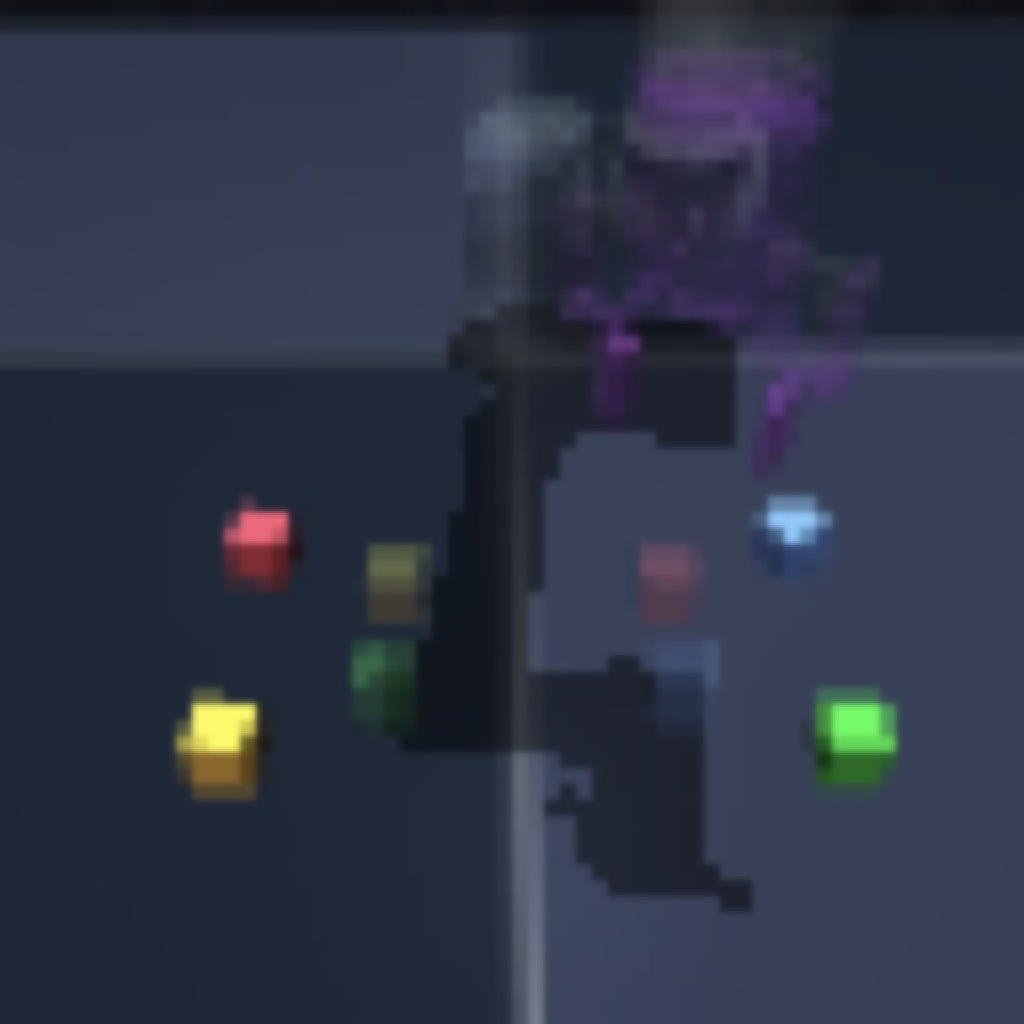}
        \caption{Pixel Observations.}
        \label{fig:pixel}
    \end{subfigure}
    \begin{subfigure}[c]{0.45\linewidth}
    \centering
    \includegraphics[width=.9\linewidth]{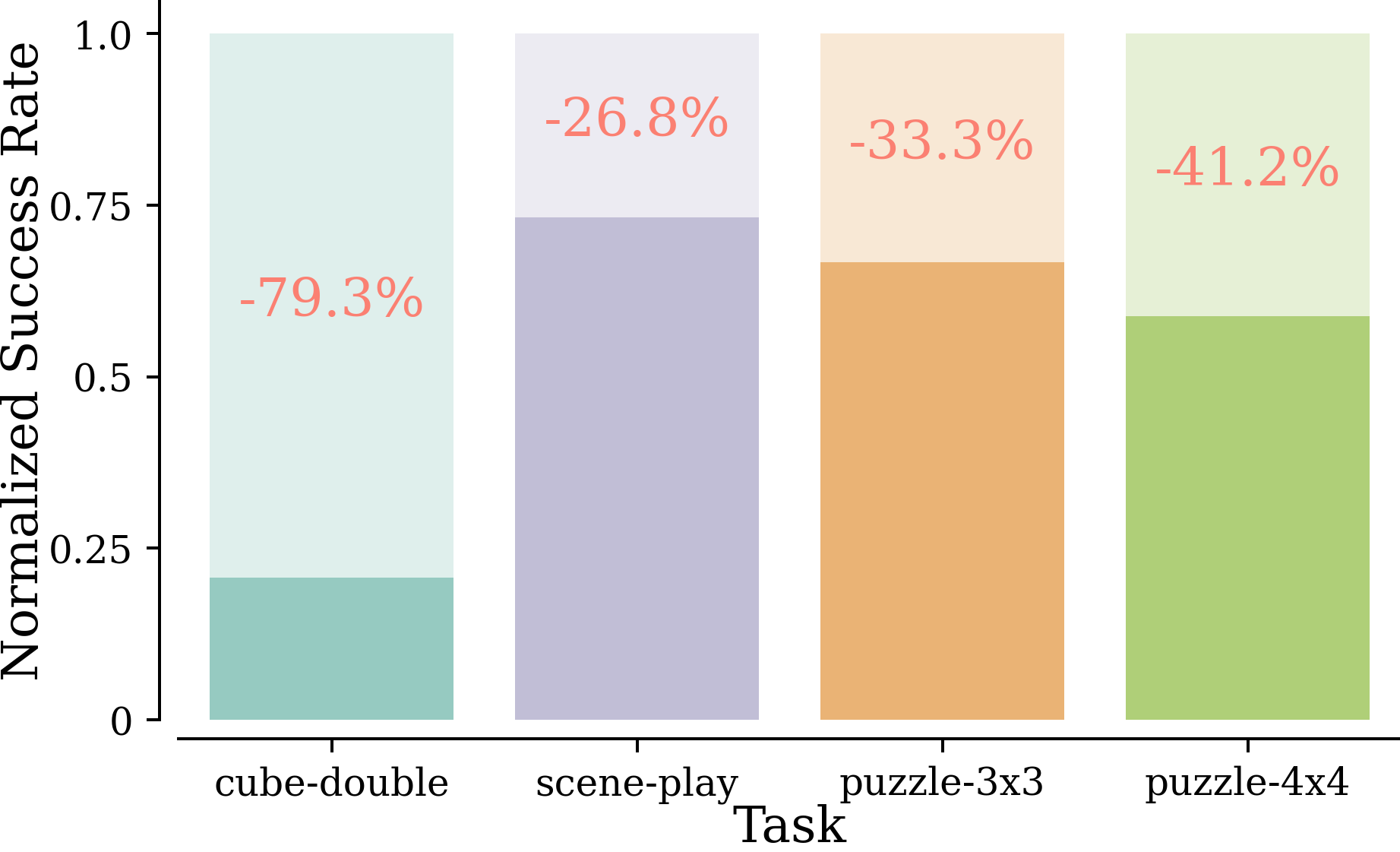}
    \caption{Performance Drop.}
    \label{fig:perf gap}
    \end{subfigure}
    \vspace{-.05in}
    \caption{\textbf{Confounding biases in offline data causes performance loss.} Left-middle: For pixel based tasks in offline RL datasets, expert actions are sampled based on true state vectors but presented with pixel observations during the offline learning stage. Right: SOTA offline RL algorithm performance drops sharply with pixel observations despite using image augmentations and strong neural encoders. }
    \label{fig:implicit bias}
    \vspace{-.2in}
\end{figure*}

Despite these progresses, all the algorithms described above build on the standard off-policy policy gradient framework \citep{sutton2018reinforcement}, which relies on the key assumption that the demonstrator's behavioral policy and the learner's target policy share the same support over the state-action domain \citep{precup2000eligibility}. When such conditions fail to hold, implicit confounding biases can be introduced into the observational data, consequently posing significant challenges for training RL systems from offline data \citep{crlsurvey,li2025confoundingdqn}. 

\begin{example}[Confounded Pixel-Based Observations]\label{exp:_bias}
\cref{fig:implicit bias} shows an example that illustrates challenges of implicit confounding bias in offline RL. Specifically, the expert actions are sampled based on true state observations, whereas the final pixel-based offline dataset pairs those actions with the corresponding image observations. The robot arm is tasked to move the colored cubes into the desired translucent positions. But pixel observations in \cref{fig:pixel}, a blurry front-view with occlusion, cannot fully capture the true state information shown in \cref{fig:state}, which underlies the expert's decisions. Learning from such datasets under the MDP assumption results in significant performance degradation despite that the same method works well in the same task with a structured state-observation space as the demonstrator does~\cite{park2025fql,dong2025valueflows}. $\hfill \blacksquare$
\end{example}

\vspace{-0.05in}
The problem of addressing confounding bias in decision-making has been studied under the rubrics of causal inference \citep{pearl2009causality} and, more recently, causal reinforcement learning \citep{crlsurvey}. Most methods explore additional prior knowledge about the underlying causal mechanisms under which the effect of the target policy is uniquely determined (i.e., identifiable) from the observational data \citep{10.5555/2074158.2074209seqbackdoor,crlsurvey}. When the causal knowledge is not available or too weak to ensure identifiability, partial identification methods could be incorporated with standard RL algorithms to reason about the optimal actions in the worst-case environments compatible with observations \citep{kallus2018confounding,zhang2019near,namkoong2020off,zhang2024eligibility,li2025confoundedshaping,li2025confoundingdqn,hess2025efficientsharpoffpolicylearning}. Functional approximators have been used to evaluate the worst-case policy return from confounded observations for environments with complex states but a simple, discrete action domain. Still, significant challenges remain in learning a robust policy from confounded, offline observations in an expressive policy class with continuous actions. \footnote{For a more detailed discussion on offline RL and causal decision-making, we refer readers to \Cref{app:related}.}

We overcome this challenge by formalizing the confounding issue in offline RL, developing a novel offline RL objective and then
employing flow-matching to model the complex state-action distributions in observed trajectories. More specifically, our contributions are as follows: (1) We introduce a novel offline RL objective that is robust against the implicit unobserved confounding in the offline data. Optimizing this objective leads to a safe policy selecting optimal actions in the worst-case environment. (2) We develop an algorithm, called \emph{Causal Flow Q-Learning} (CFQL), to optimize the newly proposed confounding robust objective over an expressive policy class, leveraging flow-matching to model the observed state-action distributions. We evaluate CFQL through extensive experiments. The result demonstrate a 120\% success rate improvement on average in 25 pixel-based offline RL tasks with some even surpassing policies using true state observations. Due to the space constraint, all proofs are provided in \Cref{app:proof}.

\textbf{Notations.} We will consistently use capital letters ($V$) to denote random variables, lowercase letters ($v$) for their values, and cursive $\1V$ to denote the their domains. We use bold capital letters ($\boldsymbol{V}$) to denote a set of random variables and let $|\boldsymbol{V}|$ denote its cardinality of set $\boldsymbol{V}$. Fix indices $i, j \in \3N$. Let $\bar{\*X}_{i:j}$ stand for a sequence $\{X_i, X_{i+1}, \dots, X_j\}$. We denote by $P(\*X)$ a probability distribution over variables $\*X$. We consistently use $P(\*x)$ as abbreviations of probabilities $P(\*X = \*x)$; so does $P(\*Y = \*y \mid \*X = \*x) = P(\*y \mid \*x)$. Finally, $\I_{\*Z = \*z}$ is an indicator function that returns $1$ if event $\*Z = \*z$ holds true; otherwise, it returns $0$. 

\section{Confounding Robust Policy Learning}
We investigate a sequential decision-making setting in which the agent chooses a sequence of actions to optimize subsequent rewards. We assume a Confounded Markov Decision Process (CMDP) to explicitly model the challenges of unobserved confounding in offline RL \citep{zhang2022can,bennett2021off,kallus2020confounding,zhang2025eligibility,li2025confoundedshaping,li2025confoundingdqn}.\footnote{See \cref{app:proof} for a detailed discussion on causal foundations and CMDPs as a relaxation to MDPs.}
\begin{definition}
\label{def:cmdp}
    A Confounded Markov Decision Process (CMDP) $\mathcal{M}$ is a tuple of $\langle \1S, \1X, \1Y, \1U, \1F, P \rangle$ where (1) $\1S, \1X, \1Y$ are, respectively, the space of observed states, actions, and rewards; (2) $\1U$ is the space of unobserved exogenous noise; (3) $\1F$ is a set consisting of the transition function $f_S: \1S \times \1X \times \1U \mapsto \1S$, behavioral policy $f_X: \1S \times \1U \mapsto \1X$, and reward function $f_Y: \1S \times \1X \times \1U \mapsto \1Y$; (4) $P$ is an exogenous distribution over the domain $\1U$.
\end{definition}

\begin{figure}[t]
\centering
		\begin{tikzpicture}
			\def\outerr{3.2}
			\def\innerr{3}
			\node[vertex] (S1) at (-1, -1) {S\textsubscript{1}};
			\node[vertex] (X1) at (0, 0) {X\textsubscript{1}};
			\node[vertex] (Y1) at (0, -2) {Y\textsubscript{1}};
			\node[vertex] (S2) at (1, -1) {S\textsubscript{2}};
			\node[vertex] (X2) at (2, 0) {X\textsubscript{2}};
			\node[vertex] (Y2) at (2, -2) {Y\textsubscript{2}};
			\node[vertex] (S3) at (3, -1) {S\textsubscript{3}};
			\node[vertex] (X3) at (4, 0) {X\textsubscript{3}};
			\node[vertex] (Y3) at (4, -2) {Y\textsubscript{3}};

			\draw[dir] (S1) to (S2);
			\draw[dir] (S1) to (X1);
			\draw[dir] (S1) to (Y1);
			\draw[dir] (X1) to (Y1);
			\draw[dir] (X1) to (S2);

			\draw[bidir, draw=betterblue] (X1) to [bend left = 30] (S2);
			\draw[bidir, draw=betterblue] (X1) to [bend left = 30] (Y1);
			\draw[bidir] (Y1) to [bend right = 30] (S2);

			\draw[dir] (S2) to (S3);
			\draw[dir] (S2) to (X2);
			\draw[dir] (S2) to (Y2);
			\draw[dir] (X2) to (Y2);
			\draw[dir] (X2) to (S3);

			\draw[bidir, draw=betterblue] (X2) to [bend left = 30] (S3);
			\draw[bidir, draw=betterblue] (X2) to  [bend left = 30] (Y2);
			\draw[bidir] (Y2) to [bend right = 30] (S3);

			\draw[dir] (S3) to (X3);
			\draw[dir] (S3) to (Y3);
			\draw[dir] (X3) to (Y3);
			\draw[dir] (S3) to node {} (4.7, -1);

			\draw[bidir, draw=betterblue] (X3) to [bend left = 30] (Y3);

			\begin{pgfonlayer}{back}
				\node[circle,fill=betterblue!65,draw=none,minimum size=2*\innerr mm] at (X1) {};
				\node[circle,fill=betterblue!65,draw=none,minimum size=2*\innerr mm] at (X2) {};
				\node[circle,fill=betterblue!65,draw=none,minimum size=2*\innerr mm] at (X3) {};
				\node[circle,fill=betterred!65,draw=none,minimum size=2*\innerr mm] at (Y1) {};
				\node[circle,fill=betterred!65,draw=none,minimum size=2*\innerr mm] at (Y2) {};
				\node[circle,fill=betterred!65,draw=none,minimum size=2*\innerr mm] at (Y3) {};
			\end{pgfonlayer}
		\end{tikzpicture}
	\caption{\textbf{Causal diagram representing the data-generating mechanisms in a Confounded Markov Decision Process.} Bi-directed arrows represent information used by the expert's policy but unobservable to the learner in the offline dataset.}
    \label{fig:_2_1_mdp}
\end{figure}
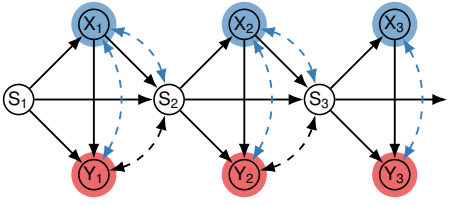

Throughout this paper, we assume the action domain $\1X$ to be continuous unless specified otherwise, while the state domain $\1S$ could be complex and continuous; the reward domain $\1Y$ is bounded in a real interval $[a, b] \subset \3R$. Consider a demonstrator agent interacting with a CMDP $\1M$, generating the off-policy data. For every time step $t = 1, \dots, T$, the environment first draws an exogenous noise $U_t$ from the distribution $P(\1U)$; the demonstrator then performs an action $X_t \gets f_X(S_t, U_t)$, receives a subsequent reward $Y_t \gets r_t(S_t, X_t, U_t)$, and moves to the next state $S_{t + 1} \gets f_S(S_t, X_t, U_t)$. The observed trajectories of the demonstrator (from the learner's perspective) are summarized as the observational distribution $P(\bar{\*X}_{1:T}, \bar{\*S}_{1:T}, \bar{\*Y}_{1:T})$.

\Cref{fig:_2_1_mdp} shows a graphical representation \citep{PCH} illustrating the generative process of the observational data in CMDPs. More specifically, solid nodes represent observed variables $X_t, S_t, Y_t$, and arrows represent the functional relationships $f_X, f_S, f_Y$ among them. By convention, exogenous variables $U_t$ are often not explicitly shown in the graph; bi-directed arrows $X_t \leftarrow \rightarrow Y_t$ and $X_t \leftarrow \rightarrow S_{t+1}$ indicate the presence of an unobserved confounder (UC) $U_t$ affecting the action, state, and reward simultaneously. 
These bi-directed arrows (highlighted in \textcolor{betterblue}{blue}) represent the unobserved confounders among action $X_t$, reward $Y_t$, and state $S_{t+1}$ in the off-policy data, violating the condition of no unobserved confounder (NUC) \citep{robbins1985some,crlsurvey}. Such violations lead to challenges in offline RL as we have seen above.

A policy $\pi$ in a CMDP $\1M$ is a decision rule $\pi(x_t \mid s_t)$ mapping from state to a distribution over action domain $\1X$. An intervention $\doo(\pi)$ is an operation that replaces the behavioral policy $f_X$ in CMDP $\1M$ with the policy $\pi$. Let $\1M_{\pi}$ be the submodel induced by intervention $\doo(\pi)$. The interventional distribution $P_{\pi}(\bar{\*X}_{1:T}, \bar{\*S}_{1:T}, \bar{\*Y}_{1:T})$ is defined as the distribution over observed variables in $\1M_{\pi}$,
\begin{align}
    &P_{\pi}(\bar{\*x}_{1:T}, \bar{\*s}_{1:T}, \bar{\*y}_{1:T}) = \nonumber \\
    &P(s_1) \prod_{t = 1}^{T} \bigg (\pi(x_t \mid s_t) \1T(s_t, x_t, s_{t+1})\1R(s_t, x_t, y_t)\bigg)
\end{align}
where the transition distribution $\1T$ and the reward distribution $\1R$ are given by, for $t = 1, \dots, T$,
\begin{align}
	&\1T(s_t, x_t, s_{t+1}) = \int_{\1U} \I_{s_{t+1} = f_S(s_t, x_t, u_t)} P(u_t),\\
    &\1R(s_t, x_t, y_t) = \int_{\1U} \I_{y_t = f_Y(s_t, x_t, u_t)} P(u_t).
\end{align}
For convenience, we write the reward function $\1R(s, x)$ as the expected value $\sum_{y} y \1R(s, x, y)$. Fix a discounted factor $\gamma \in [0, 1]$. A common objective for an agent is to optimize its cumulative return $R_t = \sum_{i = 0}^{\infty} \gamma^i Y_{t + i}$.  

\textbf{Offline RL.} In offline RL, when there is no unobserved confounder introducing spurious correlations between actions and subsequent outcomes, one can identify the parameterizations of the transition distribution $\1T$ and the reward function $\1R$ from offline data. This means that the expected return of candidate policies can be estimated from the sampling process of the offline data. An optimal policy is then obtainable by maximizing the estimated return (under proper behavioral regularization).

For example, let the state-action value function $Q_{\phi}(s, x) = \3E_{\pi}[R_t \mid S_t = s, X_t = x]$ denote the expected return of a target policy policy $\pi$ conditioning on state $s$ and action $x$ with parameter $\phi$. Let $\theta$ denote the parameter of $\pi$. The basic offline RL objective with behavior regularization is defined as \citep{wu2019behaviorregularizedofflinereinforcement,fujimoto2021a,tarasov2023rebrac}:
\begin{equation}
\begin{aligned}
    \1L(\theta) = \3E_{\substack{x, s \sim \1D,\\x'\sim\pi_\theta}}\bigg[ -Q_\phi(s,x') - \alpha \log \pi(x|s) \bigg], 
\end{aligned}\label{eq:_offline obj}
\end{equation}
where $\1D$ is the offline observations; 
and $\alpha$ are some coefficients that control the strength of the behavioral cloning (BC) regularizer. The Q-value critic $Q_{\phi}$ is trained by empirical Bellman loss on the offline dataset. Target policy is learned by applying gradient-based optimization to the objective in \Cref{eq:_offline obj}. When the NUC condition holds in the data, the target policy is guaranteed to improve and converge to the optimum under standard convexity assumptions \citep{zhang2020global}.

\textbf{Causal Reinforcement Learning.} However, the NUC condition does not always hold in real-world applications. Confounding bias can arise when the sensory capabilities of the demonstrator and the learner differ, posing challenges for offline learning algorithms. 

\begin{example}[Confounded Pixel Observations (continued)]
Consider the example in \Cref{fig:implicit bias} again, which illustrates the presence of implicit confounding bias in offline RL data. 
\Cref{fig:perf gap} shows the return of flow policies trained on pixels in comparison to the ones trained on the same task but structured state observations \citep{park2025fql,dong2025valueflows}. We see a consistent trend of performance loss across the board, with some even suffering a nearly 80\% drop. All success rate under pixel-based observations are normalized with respect to the performance under structured state observations of the same algorithm and averaged over all 5 tasks in each category from OGBench \citep{park2025ogbench}. $\hfill \blacksquare$
\end{example}
Recently, there has been a growing body of work in causal inference \citep{kallus2018confounding,zhang2019near,DBLP:conf/aaai/JoshiZB24safepolicycausal,namkoong2020off,li2025confoundingdqn} and safe reinforcement learning \citep{CQL2020Kumar,calql} to address the challenges of data bias and distribution shifts in policy learning. 
Closest to our setting, \citet{zhang2025eligibility} addressed the challenges of confounding bias in off-policy evaluation by deriving a novel \emph{Causal Bellman equation} to bound the value function from confounded observations. Specifically, the value function $Q_{\pi}(s, x)$ of a policy $\pi$ is lower bounded by a function $\underline{Q_{\pi}}(s, x)$ given by,
\begin{align}
	\underline{Q_{\pi}}(s, x) &= \mu(\neg x \mid s) \bigg (a + \gamma \min_{s'} \underline{V_{\pi}}(s') \bigg ) \label{eq:_3_lower_q_2} \\
	&+ \mu(x\mid s)\bigg (\widetilde{\1R}\Parens{s, x} + \gamma \sum_{s', x'} \widetilde{\1T}\Parens{s, x, s'} \underline{V_{\pi}}(s') \bigg )  \notag 
\end{align}
    where state value bound $\underline{V_{\pi}}(s) = \sum_{x} \pi(x \mid s) \underline{Q_{\pi}}(s, x)$ and $P(\neg x \mid s) = 1 - P(x \mid s)$; $\widetilde{\1T}$ and $\widetilde{\1R}$ are nominal transition distribution and reward function computed from the observational distribution, i.e., 
    \begin{align}
        &\widetilde{\1T}\Parens{s, x, s'} = P\Parens{S_{t+1} = s' \mid S_t = s, X_t = x}, \\
        &\widetilde{\1R}\Parens{s, x} = \3E\Brackets{Y_t \mid S_t = s, X_t = x}. \label{eq:nominal}
    \end{align}
The learner can then obtain an effective policy from confounded offline data by iteratively optimizing the above value-function lower bound over the action space \citep{li2025confoundingdqn}. However, the existing causal Bellman equation is limited to discrete actions and, consequently, cannot handle complex, multimodal continuous action distributions. The remainder of this paper will address this challenge.
\section{Causal Flow Q-Learning}
In this section, we introduce a causal offline RL objective, based on which we present Causal Flow Q-Learning. 

We note that for any policy $\pi$, the state value function is lower bounded by $\underline{V_{\pi}}(s) = \sum_{x} \pi(x \mid s) \underline{Q_{\pi}}(s, x)$, where $\underline{Q_{\pi}}(s, x)$ is given by \Cref{eq:_3_lower_q_2}. After a few simplifications, a closed-form solution can be derived, as shown next.
\begin{theorem}\label{thm:_3_1}
	For a CMDP environment $\M$ with reward signals $Y_t \in [a, b] \subseteq \3R$, fix a policy $\pi$. The state value function $V_{\pi}(s) \geq \underline{V_{\pi}}(s)$ for any state $s \in \1S$, where the lower bound $\underline{V_{\pi}}(s)$ is given by as follows,
	\begin{align}
		&\underline{V_{\pi}}(s) = \3E_{\substack{x \sim p(\cdot \mid s)\\ x', x^*\sim \pi(\cdot \mid s)}} \bigg[\I_{x \neq x'} \Big (a + \gamma \min_{s^*} \underline{Q_{\pi}}(s^*, x^*) \Big)   \notag \\
        &+ \I_{x = x'} \Big (\widetilde{\1R}\Parens{s, x} + \gamma \sum_{s'} \widetilde{\1T}\Parens{s, x, s'}\underline{Q_{\pi}}(s', x^*)\Big) \bigg] \label{eq:_lower_v}
	\end{align}
\end{theorem}

\Cref{fig:_3_backup} shows a backup diagram illustrating this update step. As in the standard Bellman optimality equation \citep{bellmanDynamicProgramming1966}, \Cref{eq:_lower_v} recursively updates the state value function using the current estimates of the optimal value function. 

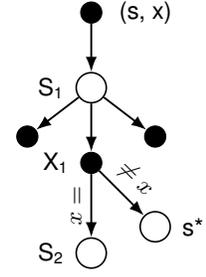
\begin{wrapfigure}[12]{r}{0.19\textwidth}
\centering
	\begin{tikzpicture}
		\def\outerr{3}
		\def\innerr{2.7}

		\node[action, label={[label distance=0.04in]0: \small (s, x)}] (start) at (0, 0) {};
		\node[vertex, label={[label distance=0.01in]180: \small S\textsubscript{1}}] (S1) at (0, -1) {};
		\node[action, label={[label distance=0.01in]180: \small X\textsubscript{1}}] (X1) at (0, -2) {};
		\node[vertex, label={[label distance=0.01in]0: \small s*}] (S2p) at (0.85, -2.85) {};
		\node[action] (X1l) at (-0.85, -1.65) {};
		\node[action] (X1r) at (0.85, -1.65) {};
		\node[vertex, label={[label distance=0.01in]180: \small S\textsubscript{2}}] (S2) at (0, -3.2) {};

		\draw[dir] (start) -- (S1);
		\draw[dir] (S1) -- (X1);
		\draw[dir] (S1) -- (X1l);
		\draw[dir] (S1) -- (X1r);
		\draw[dir] (X1) -- (S2)node [below, midway, sloped] {\small $= x$};
		\draw[dir] (X1) -- (S2p) node [above, midway, sloped] {\small $\neq x$};
	\end{tikzpicture}
	\caption{Backup diagram for causal policy gradient.}
	\label{fig:_3_backup}
\end{wrapfigure}

On the other hand, \Cref{eq:_lower_v} explicitly accounts for the off-poicy nature of the confounded observations: when the behavior policy takes the same action $x_t = x$ as the target action, the update follows standard Bellman equation and uses the next sampled state $s_t$; when the sampled action $x_t \neq x$ differs from the target, our algorithm updates, instead, using the value function associated with the next worst-case or best-case state $s^*$, corresponding to the estimation of the lower bound and upper bound respectively.

\Cref{thm:_3_1} lower bounds the expected return of a policy $\pi$ that optimizes a worst-case CMDP instance $\underline{\1M}$ compatible with the confounded observations. Maximizing the lower bound in \Cref{thm:_3_1} leads to a safe policy with a performance guarantee in the underlying environment. In causal offline RL objective, this bound is used as the Q loss in \cref{eq:_offline obj} and can be calculated with function approximators. Specifically, let $\theta$ denote the parameters of the target policy $\pi$, and $\1D$ be the finite set of offline observations. Then, the causal offline RL objective is estimated from the observational data as,
\begin{align}
    &\1L(\theta) = 
    \3E_{\substack{(s, x, y, s') \sim \1D \\ x', x^* \sim \pi_{\theta}(\cdot \mid s)}} \bigg[D(s, x, x') \Big (y + \gamma \underline{Q_{\pi_{\theta}}}(s', x^*)\Big)  \notag \\
        &+ (1 - D(s, x, x')) \Big (a + \gamma \min_{s^*} \underline{Q_{\pi_{\theta}}}(s^*, x^*) \Big) \bigg]  \label{eq:_cpg} 
\end{align}
Among the above quantities, $D(s,x,x')$ is a discriminator that approximates the indicator function $\I_{x = x'}$ at state $s$. It takes as input a state action pair, $s, x$, and returns $1$ if the policy distribution $\pi_\theta(x|s)$ agrees with offline data. $x^*$ is an overloaded notation for next step actions in general sampled from $\pi_\theta$. We will discuss the training details further next.

We incorporate the confounding robust Q-value target in \Cref{eq:_cpg} to offline RL objective, based on Flow Q-Learning (FQL) introduced by \citet{park2025fql}.
Note that the proposed causal offline RL objective is not limited to FQL, and could be incorporated into other gradient-based policy learning to mitigate the influence of confounding bias. For example, we demonstrate an alternative implementation in \cref{app:value flow} using value flows \citep{dong2025valueflows}.

Similar to FQL, our proposed augmentation, called Causal FQL (CFQL), focuses on learning two policies during training: (1) a behavioral cloning (BC) policy $\mu_{\omega}$ mimicking the nominal behavioral policy distribution $P(X_t \mid S_t)$ from the observational data; (2) a target policy $\pi_{\theta}$ deciding the learner's actions. Specifically, the BC policy $\mu_{\omega}(s, z)$ is a continuous normalizing flow \citep{lipman2022flow} taking a state $s$ and an independent noise $z$ as input. It transforms the noise $z$ into an action $x$ drawn from the nominal behavioral distribution $P(X_t \mid S_t)$ following a multi-step denoising procedure. On the other hand, the target policy $\pi_{\theta}(s, z)$ is a one-step generative model directly mapping a noise $z$ into an instance $x$ in the action domain. 

Now we describe the details of our proposed algorithm, Causal FQL (\Cref{alg:_3_cfql}). CFQL contains four main components: an ensemble of critic networks $Q_{\phi_i}$, a BC flow policy $\mu_{\omega}$, a one-step target policy $\pi_{\theta}$, and a discriminator $D_{\psi}$ differentiating the actions drawn from the BC and target policies. More specifically, Steps 4-9 train the ensemble of critic Q-networks $Q_{\phi_i},i = 1, 2, \dots, N$. Each network $Q_{\phi_i}$ is trained from the confounded observational data following the standard minimization of the one-step Bellman update (Step 8) over the average of Q ensembles $\bar{Q}_{\phi}(s',x') = \frac{1}{N}\sum_i {Q}_{\phi_i}(s',x')$.

From Steps 10 - 15, the algorithm trains the velocity field $v_{\omega}(t, s, x^t)$ of a BC flow policy $\mu_{\omega}$. It models the transformation of an independent noise $x^0$ to the observed action $x^1$ as trajectories of a free particle moving at a constant speed. The velocity field $v_{\omega}(t, s, x^t)$ is a neural network recording the gradient of the trajectory at any time step $t$. The training follows the standard flow-matching procedure \citep{lipman2022flow,lipman2024flow}.
Steps 16 - 20 train a discriminator network $D_{\psi}$ in the causal policy gradient \Cref{eq:_cpg} to differentiate the samples drawn from the BC flow $\mu_{\omega}$ and target policy $\pi_{\theta}$. Specifically, the algorithm obtains samples drawn from $\mu_{\omega}(s, z)$ and $\pi_{\theta}(s, z)$ given the observed state $s$ and an independent noise $z$ (Steps 17-19). The samples of the BC flow policy are computed from the trained velocity field following the standard Euler method (\Cref{alg:_3_bc}) 
The discriminator $\1 D_\psi$'s training objective is then given by,
\begin{align}
    \1 L_{\operatorname{Disc}}(\psi) = &\mathbb{E}_{s\sim \1 D, z \sim \1 N(0, I_d)}\big[ \log D_\psi(s, \mu_{\omega}(s, z)) \notag \\
    &+  \log(1 - \1D_\psi(s,\pi_{\theta}(s, z)))\big]
\end{align}
Minimizing the above discriminator loss is equivalent to solving a binary classification problem where class $1$ represents the action $x$ is from the BC flow policy $\mu_\omega$ while class $0$ represents that the action is sampled from the one-step target policy $\pi_\theta$ given the current state $s$.

Finally, CFQL trains a one-step target policy $\pi_{\theta}$ optimizing the worst-case return compatible with the confounded observations in Steps 21 - 24. 
The key challenge here is to obtain a reliable estimation for the worst-case value funciton $\min_{s^*} \underline{Q_{\pi_{\theta}}}(s^*, x^*)$ in \Cref{eq:_cpg} without being overly pessimistic. To address this challenge, instead of taking the minimum over the whole state space or the whole batch of data~\citep{li2025confoundingdqn}, we propose to learn ensembles of Q-networks and take the minimum over the learned ensembles to simulate the worst case scenario.

\begin{algorithm}[H]
	\caption{Causal Flow Q-Learning (\texttt{Causal FQL})}
	\label{alg:_3_cfql}
	\setlength{\textfloatsep}{0pt}
	\begin{algorithmic}[1]
    \STATE {\bfseries Input:} Offline dataset $\1D$.
        \WHILE{not converged}
            \STATE Sample batch $\{(s,x,y,s')\}\sim \1D$

            \item[] \textcolor{betterblue}{$\triangledown$ Train critic $Q_{\phi_i}$}
            \FOR{$i = 1, \dots, N$} 
            \STATE $z \sim \1N(0, I_d)$
            \STATE $x' \gets \pi_{\theta}(s', z)$,
            \STATE Update $\phi_i$ to minimize: \\ $\3E\big[(Q_{\phi_i}(s,x) - y - \gamma \bar{Q}_{\phi}(s',x') )^2 \big]$
            \ENDFOR

            \item[] \textcolor{betterblue}{$\triangledown$ Train BC flow policy $\mu_{\omega}$}
            \STATE $x^0 \sim \1N(0, I_d)$ 
            \STATE $x^1 \gets x$ 
            \STATE $t \sim \text{Unif}(0, 1)$
            \STATE $x^t \gets (1 - t)x^0 + x^1$
            \STATE Update $\omega$ to minimize $\3E\big[ \lVert v_{\omega}(t, s,x^t) - (x^1 - x^0)\rVert^2 \big]$

            \item[] \textcolor{betterblue}{$\triangledown$ Train discriminator $D_{\psi}$}
            \STATE $z \sim \1N(0, I_d)$
            \STATE $x \gets \mu_{\omega}(s, z)$
            \STATE $x' \gets \pi_{\theta}(s, z)$
            \STATE Update $\psi$ to minimize:\\ $\mathbb{E}\big[ \log(1 - \1D_\psi(s,x)) + \log D_\psi(s, x') \big] $
            
            \item[] \textcolor{betterblue}{$\triangledown$ Train one-step target policy $\pi_{\theta}$}
            \STATE $z \sim \1N(0, I_d)$
            \STATE $x \gets \pi_{\theta}(s, z)$
            \STATE Update $\theta$ to minimize: \\ $\3E\big[-\underline{Q_{*}}(s, x) +\alpha \lVert x - \mu_{\omega}(s, z) \rVert^2 \big ]$
        \ENDWHILE
    \STATE {\bfseries Return} one-step flow policy $\pi_\theta$.
	\end{algorithmic}
\end{algorithm}
\begin{algorithm}[H]
	\caption{BC Flow Policy $\mu_{\omega}(s, z)$}
	\label{alg:_3_bc}
	\setlength{\textfloatsep}{0pt}
	\begin{algorithmic}[1]
    \STATE {\bfseries Input:} Velocity field $v_{\omega}(t, s, z)$, time step $M$.
        \FOR{$t = 0, 1, \dots, M - 1$}
        \STATE $z \gets z + v_{\omega}(t/M, s, z)/M$
        \ENDFOR
    \STATE {\bfseries Return} sampled action $z$
	\end{algorithmic}
\end{algorithm}

\begin{table*}[ht]
\caption{\textbf{Aggregated offline evaluation results}. We report aggregated mean success rate and mean deviation across different task groups. Causal-FQL consistently outperforms all baselines across the task groups. When comparing against the original FQL, the confounding robust Causal-FQL achieves a 1.2$\times$ success rate improvement on average.}
\centering
\resizebox{\textwidth}{!}{%
\begin{tabular}{lcccccccc}
\toprule
& \multicolumn{3}{c}{Gaussian Policies} & \multicolumn{4}{c}{Flow Policies} \\
\cmidrule(lr){2-4} \cmidrule(lr){5-8}
& IQL & ReBRAC & FBRAC & IQN & IFQL & FQL & \texttt{Causal-FQL} \\
\midrule
\texttt{visual-cube-single-play-singletask (5 tasks)} & $64 \pm 17$ & $\mathbf{82 \pm 7}$ & $64 \pm 10$ & $58 \pm 7$ & $44 \pm 9$ & $65 \pm 12$ & $\mathbf{81 \pm 6}$ \\
\texttt{visual-cube-double-play-singletask (5 tasks)} & $\mathbf{11 \pm 6}$ & $1 \pm 1$ & $2 \pm 1$ & $1 \pm 0$ & $2 \pm 2$ & $6 \pm 1$ & $\mathbf{11 \pm 3}$ \\
\texttt{visual-scene-play-singletask (5 tasks)}  & $26 \pm 5$ & $28 \pm 5$ & $11 \pm 1$ & $\mathbf{41 \pm 6}$ & $21 \pm 2$ & $\mathbf{41 \pm 4}$ & $\mathbf{43 \pm 1}$ \\
\texttt{visual-puzzle-3x3-play-singletask (5 tasks)} & $2 \pm 3$ & $20 \pm 1$ & $1 \pm 0$ & $19 \pm 1$ & $21 \pm 0$ & $20 \pm 1$ & $\mathbf{26 \pm 2}$ \\
\texttt{visual-puzzle-4x4-play-singletask (5 tasks)}  & $0 \pm 0$ & $5 \pm 1$ & $0 \pm 0$ & $\mathbf{13 \pm 3}$ & $4 \pm 7$ & $11 \pm 3$ & $\mathbf{13 \pm 2}$ \\
\midrule
\textbf{Normalized Mean} & 0.72 & 0.95 & 0.55 & 0.92 & 0.65 & 1.0 & \textbf{1.21} \\
\bottomrule
\end{tabular}%
}
\label{tab:aggregated_results}
\end{table*}

With the discriminator $D_\psi$, the value function evaluation in the causal policy gradient can be approximated as,
\begin{align}
    \underline{Q_{*}}(s, x) &= D_\psi(s,x)\bar{Q}_{\phi}(s,x) \notag \\
    &\phantom{xxx}+ (1 -D_\psi(s,x))\min_i Q_{\phi_i}(s,x)
    \label{eq:approx causal q obj}
\end{align}

And the one-step target policy is trained with optimizing the following objective function:
\begin{align}
    \1 L_\pi(\theta) &= \3E_{s\sim \1 D, x\sim\pi_\theta}\big[ - \underline{Q_{*}}(s, x)\big] +\alpha \1L_{\operatorname{Distill}}(\theta)
\end{align}
where $\alpha$ is a hyperparameter controlling the strength of the target policy regularizer; and $\1L_{\operatorname{Distill}}(\theta)$ is the distillation loss following the definition of \citep{park2025fql}, i.e., 
\begin{align}
    \1 L_{\operatorname{Distill}}(\omega) &= \3 E_{\substack{s\sim\1D,\\z\sim \1N(0, I_d)}}\big[ \Vert \pi_\theta(s,z) - \mu_\omega(s,z) \Vert_2^2 \big]
\end{align}
The CFQL proposed is similar to the original FQL, with an additional discriminator training step and regularization to improve worst-case value function estimation. This means that the gradient-based policy optimization could be consistently performed using the standard automatic differentiation framework \citep{baydin2018auto}.

\section{Experiments}

\begin{table*}[ht]
\centering
\label{tab:full offline}
\scriptsize 
\setlength{\tabcolsep}{3pt} 
\renewcommand{\arraystretch}{1.1} 
\newcommand{\tb}[1]{\mathbf{#1}} 
\caption{\textbf{Full offline evaluation results.} \texttt{Causal-FQL} outperforms the vanilla FQL on all domains and achieves the best or near-best results in 19 out of 25 domains. Due to high computation requirements, we report results averaged over 4 seeds and bold values within 95\% of the best performance of each domain and reuse results from prior work when available \citep{park2025fql,dong2025valueflows}.}
\resizebox{\textwidth}{!}{%
\begin{tabular}{l cc ccccc}
\toprule
& \multicolumn{2}{c}{Gaussian Policies} & \multicolumn{5}{c}{Flow Policies} \\
\cmidrule(lr){2-3} \cmidrule(lr){4-8}
& IQL & ReBRAC & FBRAC & IQN & IFQL & FQL & \texttt{Causal-FQL} \\
\midrule

\ttfamily visual-cube-single-play-singletask-task1-v0 (*) & $70 \pm 12$ & $\tb{83 \pm 6}$ & $55\pm 8$ & $5\pm4$ & $49\pm 7$ & $81\pm 12$ & $\tb{85 \pm 6}$\\
\ttfamily visual-cube-single-play-singletask-task2-v0 & $59\pm13$ & $\tb{85\pm6}$ & $70\pm11$ & $58\pm9$ & $45\pm11$ & $58\pm 12$ & $\tb{81 \pm 5}$\\
\ttfamily visual-cube-single-play-singletask-task3-v0 & $64\pm35$ & $\tb{88\pm4}$ & $81\pm4$ & $83\pm11$ & $67\pm6$ & $73\pm16$ & $\tb{84 \pm 5}$\\
\ttfamily visual-cube-single-play-singletask-task4-v0 & $68\pm12$ & $\tb{78\pm9}$ & $56\pm19$ & $\tb{78\pm3}$ & $34\pm13$ & $63\pm8$ & $\tb{74 \pm 5}$\\
\ttfamily visual-cube-single-play-singletask-task5-v0 & $59\pm14$ & $\tb{76\pm9}$ & $59\pm10$ & $64\pm9$ & $27\pm10$ & $49\pm12$ & $\tb{79 \pm 7}$\\
\midrule

\ttfamily visual-cube-double-play-singletask-task1-v0 (*) & ${34 \pm 23}$ & $4 \pm 4$ & $6 \pm 2$ & $4 \pm 1$ & $8 \pm 6$ & $23 \pm 4$ & $\tb{43 \pm 11}$\\
\ttfamily visual-cube-double-play-singletask-task2-v0 & $\tb{2 \pm 1}$ & $0 \pm 0$ & $2 \pm 2$ & $0 \pm 0$ & $0 \pm 0$ & $0 \pm 0$ & $\tb{2 \pm 1}$\\
\ttfamily visual-cube-double-play-singletask-task3-v0 & $\tb{7 \pm 4}$ & $2 \pm 2$ & $2 \pm 1$ & $0 \pm 0$ & $1 \pm 1$ & $4 \pm 2$ & ${4 \pm 2}$\\
\ttfamily visual-cube-double-play-singletask-task4-v0 & $\tb{1 \pm 1}$ & $0 \pm 0$ & $0 \pm 0$ & $0 \pm 0$ & $0 \pm 0$ & $0 \pm 0$ & $\tb{1 \pm 1}$\\
\ttfamily visual-cube-double-play-singletask-task5-v0 & $\tb{11 \pm 2}$ & $0 \pm 0$ & $0 \pm 0$ & $1 \pm 1$ & $2 \pm 1$ & $4 \pm 1$ & ${4 \pm 2}$\\
\midrule

\ttfamily visual-scene-play-singletask-task1-v0 (*) & $\tb{97 \pm 2}$ & $\tb{98 \pm 4}$ & $46 \pm 4$ & $\tb{95 \pm 2}$ & $86 \pm 10$ & $\tb{98 \pm 3}$ & $\tb{100 \pm 0}$ \\
\ttfamily visual-scene-play-singletask-task2-v0 & $21 \pm 16$ & $30 \pm 15$ & $0 \pm 0$ & $79 \pm 15$ & $0 \pm 0$ & ${86 \pm 8}$ & $\tb{92 \pm 3}$ \\
\ttfamily visual-scene-play-singletask-task3-v0 & $12 \pm 9$ & $10 \pm 7$ & $10 \pm 3$ & $\tb{31 \pm 14}$ & $19 \pm 2$ & $22 \pm 6$ & ${22 \pm 1}$\\
\ttfamily visual-scene-play-singletask-task4-v0 & $1 \pm 0$ & $0 \pm 0$ & $0 \pm 0$ & $0 \pm 0$ & $0 \pm 0$ & $1 \pm 1$ & $\tb{2 \pm 1}$\\
\ttfamily visual-scene-play-singletask-task5-v0 & $\tb{0 \pm 0}$ & $\tb{0 \pm 0}$ & $\tb{0 \pm 0}$ & $\tb{0 \pm 0}$ & $\tb{0 \pm 0}$ & $\tb{0 \pm 0}$ & $\tb{0 \pm 0}$\\
\midrule

\ttfamily visual-puzzle-3x3-play-singletask-task1-v0 (*) & $7 \pm 15$ & $88 \pm 4$ & $7 \pm 2$ & $84 \pm 1$ & $\tb{100 \pm 0}$ & $94 \pm 1$ & $\tb{98 \pm 1}$\\
\ttfamily visual-puzzle-3x3-play-singletask-task2-v0 & $0 \pm 0$ & $\tb{12 \pm 1}$ & $0 \pm 0$ & $6 \pm 2$ & $0 \pm 0$ & $0 \pm 0$ & ${1 \pm 0}$\\
\ttfamily visual-puzzle-3x3-play-singletask-task3-v0 & $0 \pm 0$ & $1 \pm 1$ & $0 \pm 0$ & $1 \pm 0$ & $2 \pm 1$ & $0 \pm 0$ & $\tb{3 \pm 2}$\\
\ttfamily visual-puzzle-3x3-play-singletask-task4-v0 & $1 \pm 1$ & $0 \pm 1$ & $0 \pm 0$ & $3 \pm 0$ & $1 \pm 0$ & $5 \pm 4$ & $\tb{8 \pm 3}$\\
\ttfamily visual-puzzle-3x3-play-singletask-task5-v0 & $0 \pm 0$ & $0 \pm 0$ & $0 \pm 0$ & $1 \pm 0$ & $0 \pm 0$ & $1 \pm 2$ & $\tb{19 \pm 4}$\\
\midrule

\ttfamily visual-puzzle-4x4-play-singletask-task1-v0 (*) & $0\pm0$ & $26\pm6$ & $0\pm 0$ & $21\pm3$ & $8\pm 15$ & $33\pm 6$ & $\tb{37 \pm 4}$ \\
\ttfamily visual-puzzle-4x4-play-singletask-task2-v0 & $0\pm0$ & $0\pm0$ & $1\pm1$ & $\tb{14\pm2}$ & $1\pm1$ & $1\pm2$ & ${4 \pm 1}$\\
\ttfamily visual-puzzle-4x4-play-singletask-task3-v0 & $0\pm0$ & $0\pm0$ & $0\pm0$ & $9\pm3$ & $9\pm15$ & $16\pm 6$ & $\tb{18 \pm 4}$\\
\ttfamily visual-puzzle-4x4-play-singletask-task4-v0 & $0\pm0$ & $1\pm1$ & $0\pm0$ & $\tb{12\pm4}$ & $2\pm2$ & $2\pm1$ & $\tb{6 \pm 2}$ \\
\ttfamily visual-puzzle-4x4-play-singletask-task5-v0 & $0\pm0$ & $0\pm0$ & $1\pm1$ & $\tb{11\pm4}$ & $0\pm0$ & $1\pm 1$ & ${1 \pm 1}$ \\
\bottomrule
\end{tabular}
}
\end{table*}

In this section, we evaluate the newly proposed Causal Flow Q-Learning on 25 OGBench~\citep{park2025ogbench} visual tasks. We aim to answer the following questions:
\begin{enumerate}[label=\textbf{Q\arabic*:}, leftmargin=23pt, topsep=0pt, parsep=0pt, itemsep=5pt]
    \item How does \texttt{Causal FQL} improve upon the vanilla FQL on pixel-based tasks?
    \item How sample efficient is \texttt{Causal FQL} on offline-to-online pixel-based tasks?
    \item What hyper-parameters matter for \texttt{Causal FQL}?
\end{enumerate}

\paragraph{Experiment Setup.} We evaluate our methods on offline RL tasks from the recently proposed OGBench~\citep{park2025ogbench}. As we consider solving the confounding biases in pixel-based offline RL tasks ()
\cref{fig:exp tasks}, previous benchmarks like D4RL~\citep{Fu2020D4RLDF} that only support structured state observations are not within our considerations. OGBench is originally proposed as a goal conditioned RL benchmark, thus we use its -\texttt{singletask} variants for compatibility with standard offline RL frameworks. More specifically, we choose manipulation tasks due to their multi-modal nature that is suitable for demonstration of flow based policy. For all tasks, we provide agents $64\times 64\times 3$ RGB image observations (e.g., \cref{fig:pixel}) with frame stacking and random crop as image augmentations. To maintain a fair comparison, we use similar sized value/policy/actor networks for all methods and use the image encoder from IMPALA~\citep{espeholtIMPALAScalableDistributed2018a} in all tasks. For the full setup details, see \Cref{app:exp detail}.

\begin{figure}[htbp]
    \centering
    \hfill
    \begin{subfigure}[b]{0.19\linewidth}
        \centering
        \includegraphics[width=\linewidth]{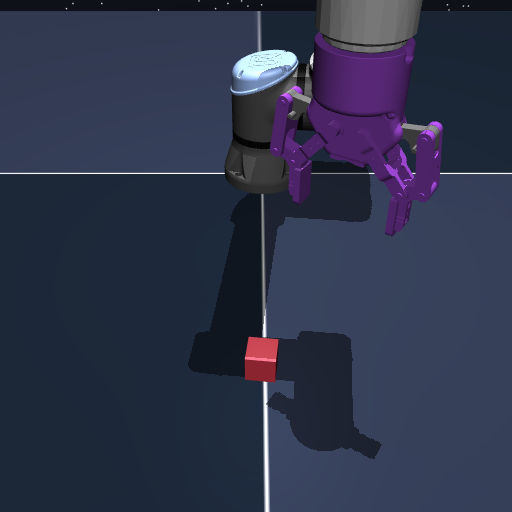}
    \end{subfigure}
    \hfill
    \begin{subfigure}[b]{0.19\linewidth}
        \centering
        \includegraphics[width=\linewidth]{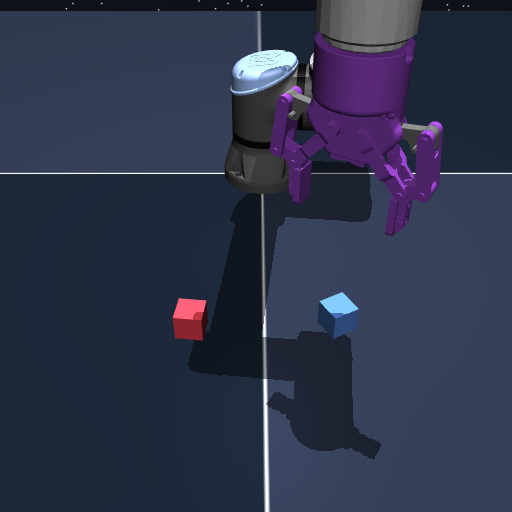}
    \end{subfigure}
    \hfill
    \begin{subfigure}[b]{0.19\linewidth}
        \centering
        \includegraphics[width=\linewidth]{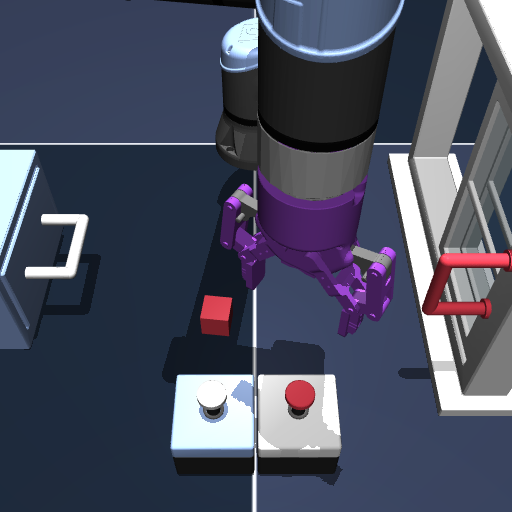}
    \end{subfigure}
    \hfill
    \begin{subfigure}[b]{0.19\linewidth}
        \centering
        \includegraphics[width=\linewidth]{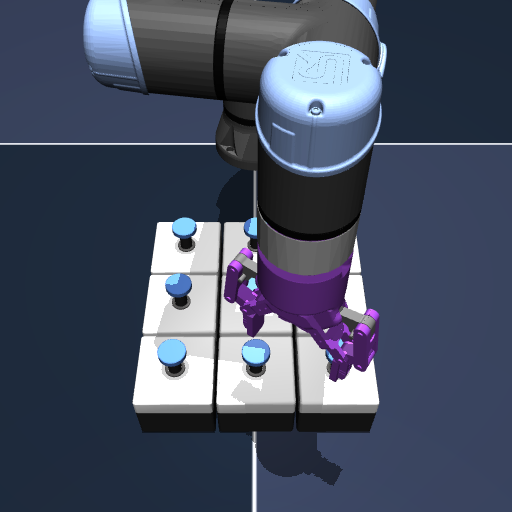}
    \end{subfigure}
    \hfill
    \begin{subfigure}[b]{0.19\linewidth}
        \centering
        \includegraphics[width=\linewidth]{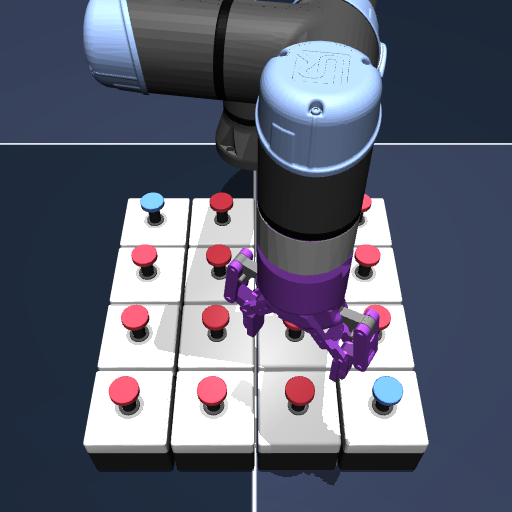}
    \end{subfigure}
    \hfill
    \label{fig:exp tasks}
    \caption{\textbf{Pixel-based OGBench Tasks}.}
\end{figure}

\paragraph{Baselines.} We use the following algorithms representing different policy extraction strategies as baselines to the offline learning tasks. Due to high computational requirements of visual tasks, we mainly use baseline algorithms that have demonstrated consistent performance in prior work~\citep{dong2025valueflows,park2025fql}.
We choose IQL~\citep{kostrikov2021iql} and ReBRAC~\citep{tarasov2023rebrac} as representatives for the class of algorithms learning scalar Q-values with a Gaussian policy. And we choose IQN~\citep{pmlr-v80-dabney18aiqn} as distributional Q-values with flow policy baselines, FBRAC, IFQL, and FQL~\citep{park2025fql} as scalar Q-values critic flow policy baselines. FBRAC and IFQL are flow policy based adoption to the original algorithm of Diffusion-QL (DQL) \citep{wang2023diffusion} and IDQL~\citep{idql}, respectively. 

For offline-to-online experiments, we mainly use ReBRAC and FQL as baselines because of their relatively strong performance than other baselines in the offline experiments. 

\paragraph{Evaluation.} For offline experiments, we run 500K steps training for all baselines and \texttt{Causal-FQL} proposed, averaging results over 4 seeds and present the results with mean and standard deviations. In the tables, we highlight entries that are within 95\% best performance on each task, following~\citep{park2025fql}. For offline-to-online experiments, we use the same setup during the offline stage and extend training to the online stage for an additional 500K steps. See \Cref{app:exp detail,app:implement detail} for implementation and experiment details.


\subsection{Offline Performance}
\begin{figure*}[t]
\centering
    \hfill
	\begin{subfigure}{0.33\linewidth}\centering
		\includegraphics[width=\linewidth]{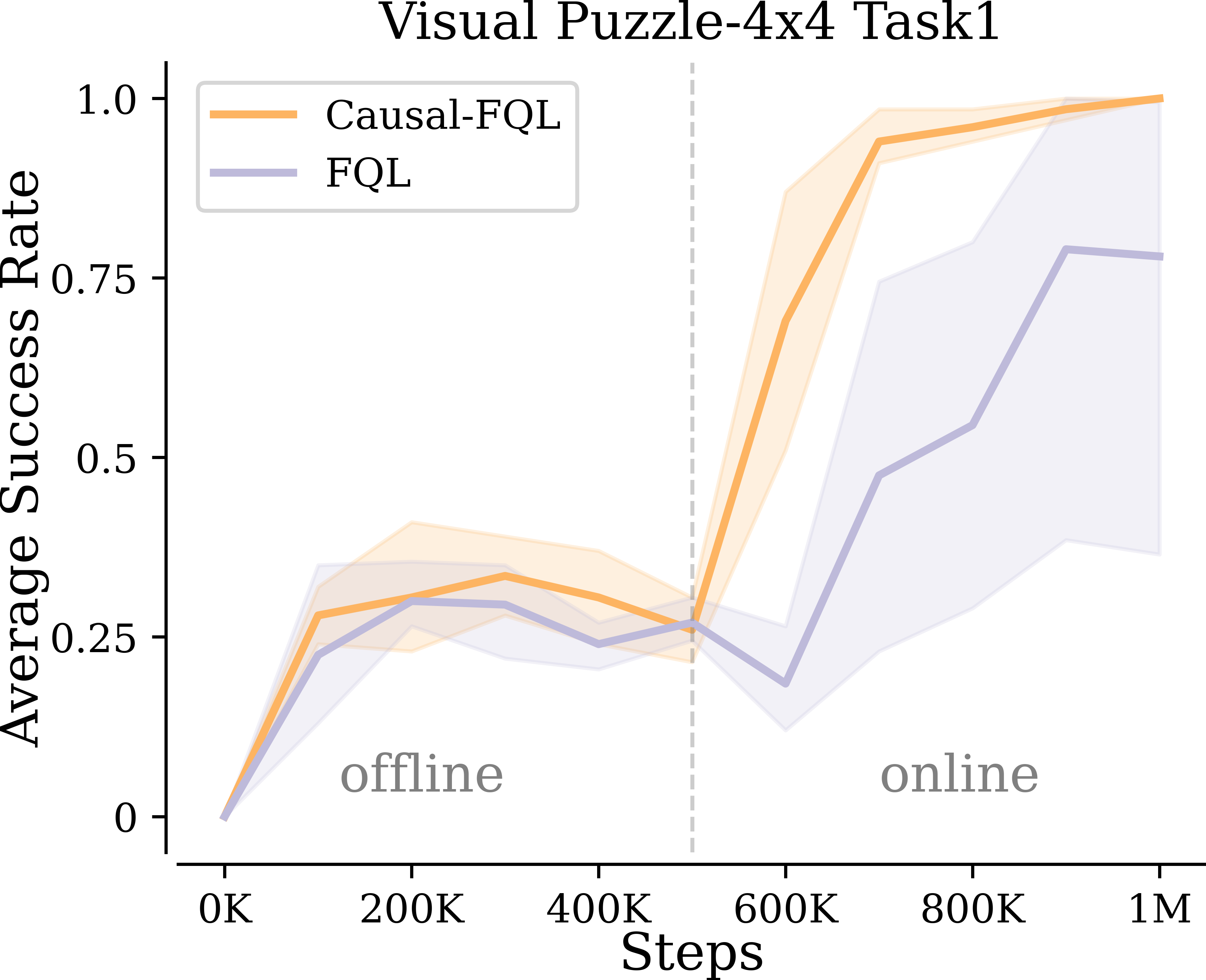}
		\label{fig:_5_2_puzzle}
	\end{subfigure}
    \hfill
    \begin{subfigure}{0.33\linewidth}\centering
		\includegraphics[width=\linewidth]{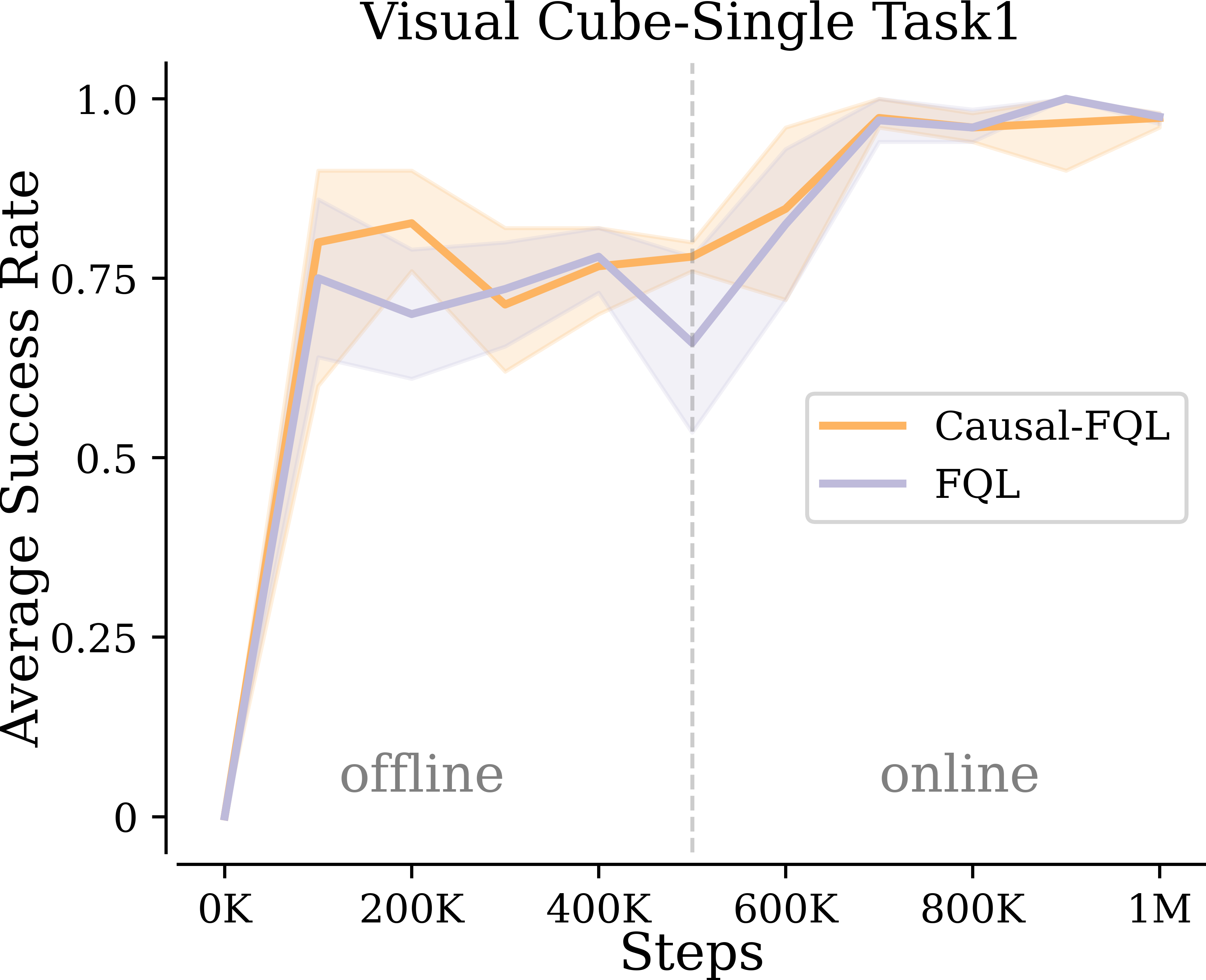}
		\label{fig:_5_2_single}
	\end{subfigure}
    \hfill
    \begin{subfigure}{0.33\linewidth}\centering
		\includegraphics[width=\linewidth]{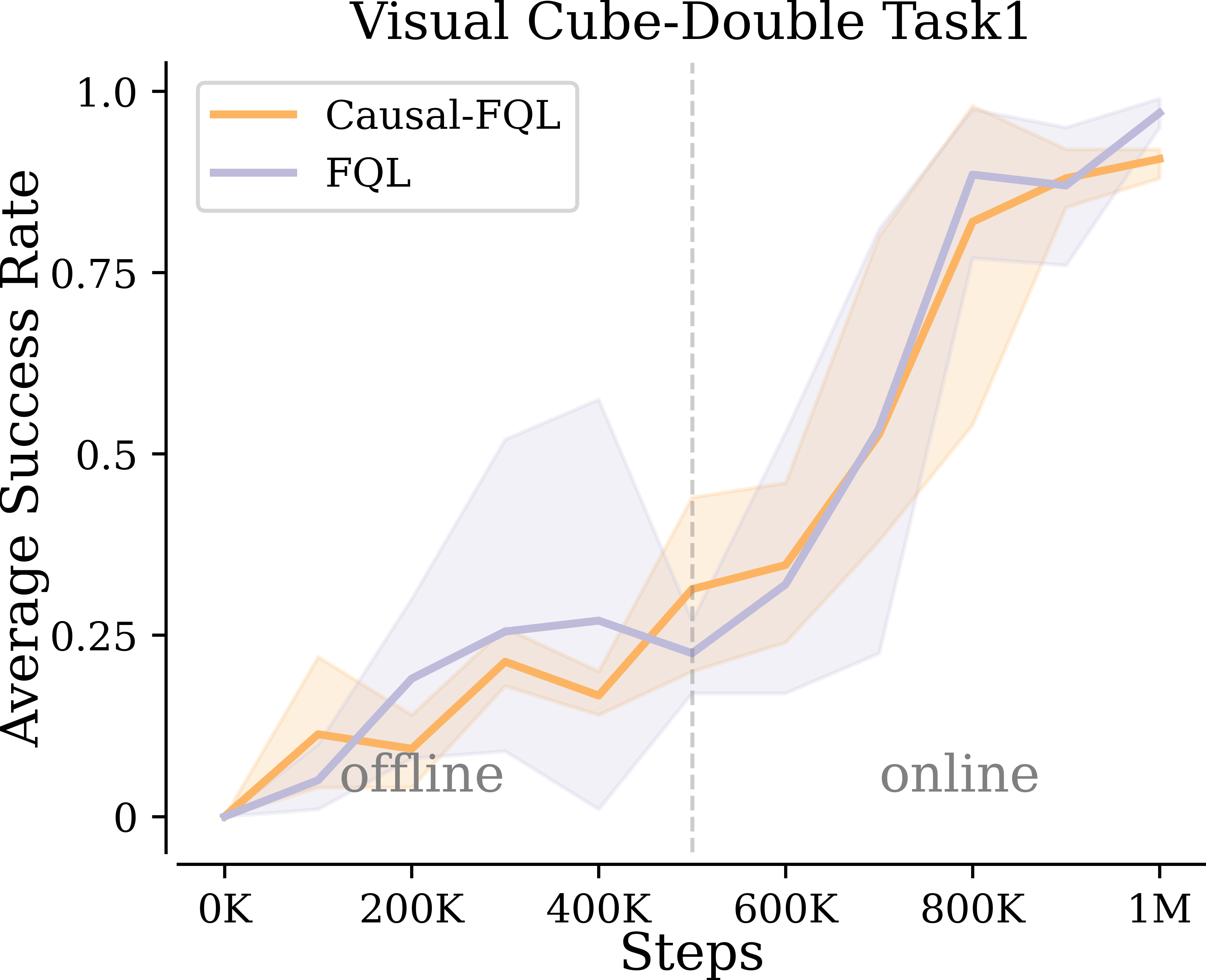}
		\label{fig:_5_2_double}
	\end{subfigure}
    \hfill\null
	\caption{\textbf{Offline-to-online evaluation.} Causal-FQL converges to near optimal success rate in three representative tasks selected from task types that are not solved near-optimal in offline evaluation. All results are averaged over 4 seeds.}
    \label{fig:_5_2_offon}
\end{figure*}
From aggregated results in \Cref{tab:aggregated_results}, \texttt{Causal-FQL} (CFQL) achieves the best or near-best performance in all task groups, surpassing the causal-unaware vanilla FQL by over 20\%. Notably, CFQL is able to match the prior state-of-the-art Gaussian baseline (ReBRAC) performance in \texttt{cube-single} in which vanilla FQL falls far behind. \Cref{tab:full offline} shows the full offline result in individual task of each group. CFQL is able to obtain the best or near-best performance in 19 out of 25 tasks. Especially in visual-cube-double-play-singletask-task1, CFQL nearly doubles the performance of vanilla FQL and in visual-puzzle-3x3-play-singletask-task5, our method is the only one that obtains non-near-zero success rate. We find that CFQL performs at least as well as FQL, empirically verifying that the confounding robust objective function is an informative lower bound on optimal Q-values. There are also tasks in which both FQL and CFQL don't perform well. We hypothesize that the biggest challenges pertain in those tasks may not be the confounding biases but rather on other dimensions of offline RL, like representation learning or simply due to limited state-action space coverage of the offline data.

\subsection{Offline-to-Online Performance}
We also study the offline to online finetuning performance of Causal-FQL. Overall, Causal-FQL achieves strong fine-tuning performance over vanilla FQL. 
It brings 8\% improvement on average over the prior SOTA, FQL's fine-tuning performance, suggesting wide applicability of Causal-FQL in both offline and offline-to-online tasks. On \texttt{visual-puzzle-4x4}, Causal-FQL is also the first algorithm in the literature to fully solve task1. \Cref{fig:_5_2_offon} shows the full evaluation curves during training.

A notable difference between Causal-FQL and prior offline RL work is that we find it necessary to use a different objective during the online phase. Recall that in \cref{fig:implicit bias}, confounding biases exist in the offline dataset because the data collection policy (expert) is operating in the structured state observation space while the learner's policy is operating in the pixel observation space. The actions from the expert depend on information that may not be fully observable in the pixel space, giving rise to the confounding biases. But when it comes to online fine-tuning, the data collection policy is the learner's policy itself and there is no observation space mismatch. Thus, we should switch back to the vanilla FQL objective function during online fine-tuning in principle. 

However, in implementation, the standard practice is to gradually replace the offline trajectories in the replay buffer with online experiences so that the training doesn't suffer from a sudden distribution shift. As a result, the actual data provided to the agent during online phase is a mixture of confounded and confounding free data. The objective function selection is subject to tuning for specific tasks. In \texttt{visual-cube-single} and \texttt{visual-cube-double}, we find it better to fully switch back to FQL objective during online phase while in \texttt{visual-puzzle-4x4}, we find it better to use balanced sampling from both online and offline replay buffer, and use the corresponding Causal-FQL/FQL objective for each half of the batch. We believe a finer control of the objective used during online fine-tuning for Causal-FQL could bring further performance and sample efficiency improvement. 

\subsection{Hyperparameter Tuning}
\begin{figure}[t]
    \centering
    \hfill
    \begin{subfigure}{0.49\linewidth}\centering
        \includegraphics[width=\linewidth]{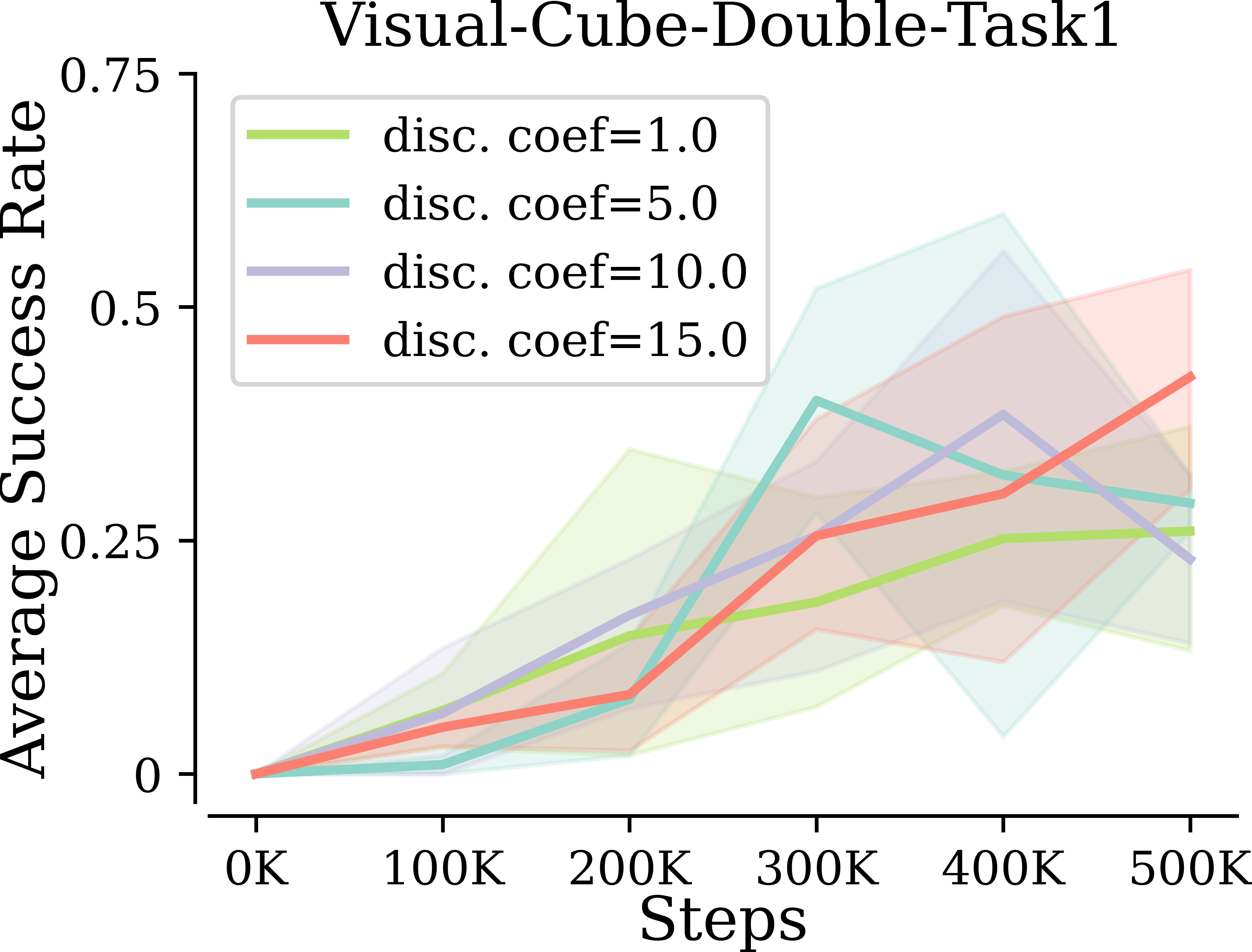}
        \label{fig:disc coef}
    \end{subfigure}
    \hfill
    \begin{subfigure}{0.49\linewidth}\centering
        \includegraphics[width=\linewidth]{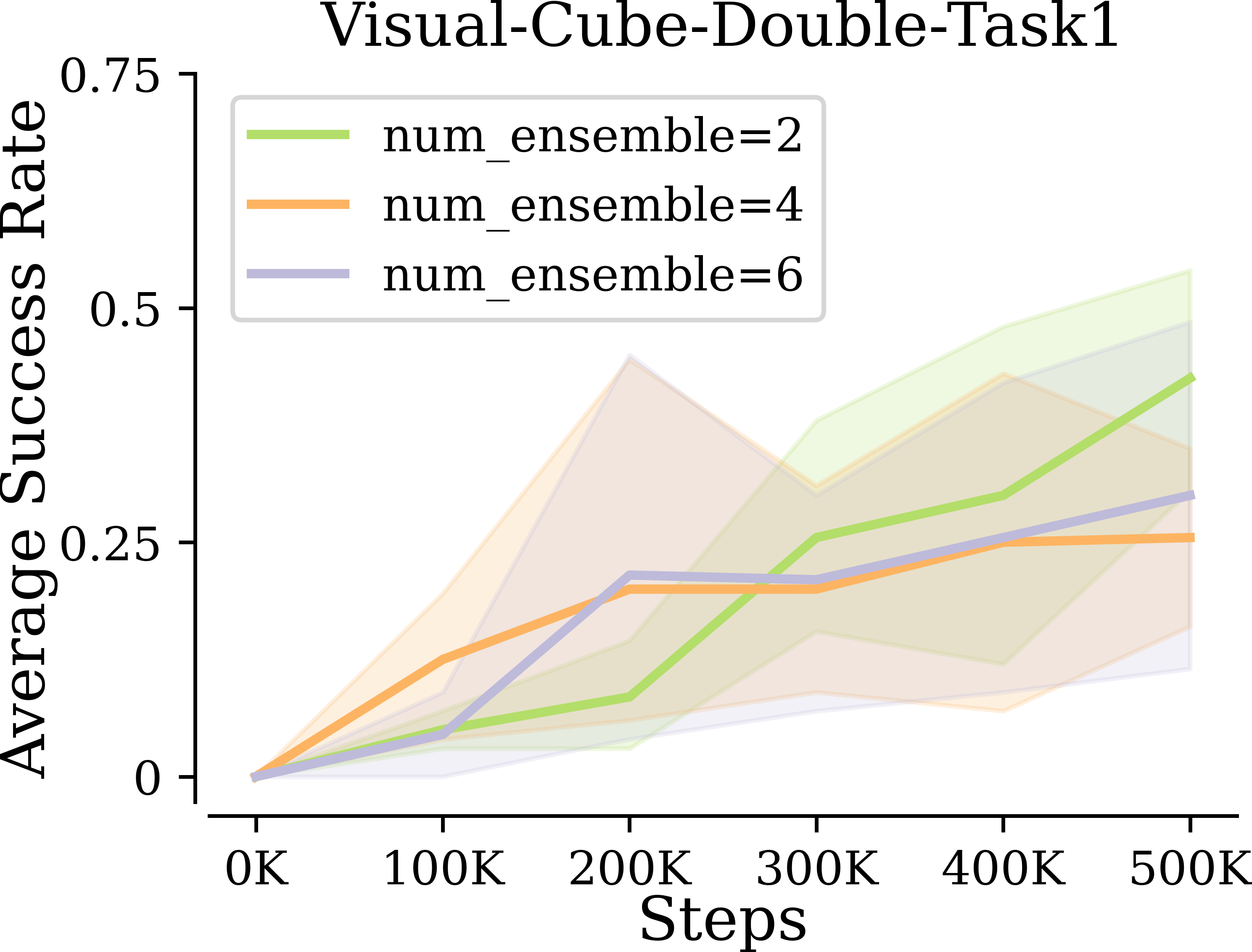}
        \label{fig:num ensemble}
    \end{subfigure}
    \hfill\null
    \caption{\textbf{A suitably tuned discriminator loss coefficient and number of critic ensembles boosts the performance.}}
    \label{fig:hyper param}
\end{figure}
In this section, we study how to tune the two fundamental components of the confounding robust Q-value objective (\Cref{eq:approx causal q obj}), the discriminator and the Q-value ensembles. We mainly tuned the hyperparameters in task 1 of \texttt{visual-cube-double} for its moderate difficulty. For each task group, we further tune the initial discriminator loss coefficient for task 1. Here we present the major findings. Detailed hyperparameters are presented in \Cref{app:exp detail}.

\paragraph{Discriminator loss coefficient.} 
We sweep the discriminator loss coefficient over $\{1,5,10,15\}$. \Cref{fig:disc coef} shows the full success rate curve over different coefficients. We notice that for overly small coefficients, the method is close to the original FQL and usually results in unstable convergence towards the end of the offline training session. Similar patterns are observed in other environments as well.

\paragraph{Number of the Q-network Ensembles.} We vary the number of Q-value ensembles, using the values $\{2, 4, 6\}$, while keeping all other parameters set to the optimal values identified in previous tuning. Our findings indicate that, in general, fewer ensembles yield better results. This may be due to the tendency to take the minimum from a larger number of ensembles, which can lead to over-pessimism in the results.


\section{Conclusion} 
We study the presence of implicit confounding biases in pixel-based offline reinforcement learning tasks and introduce Causal-FQL, a novel flow-based offline agent that is robust to such confounding factors. Theoretically, we establish a novel objective for the causal policy gradient that optimizes the target policy's performance in the worst-case environment. In practice, we propose a straightforward yet high-performance approximation to the confounding-robust offline RL objective, utilizing deep ensembles of discriminators and critics. Extensive experiments show that Causal-FQL outperforms prior state-of-the-art methods in both offline and challenging visual control tasks that transition from offline to online settings. Causal-FQL is the first method to fully solve certain tasks in the OGBench visual benchmark. Finally, we emphasize the significance of principled causal analysis in RL problems; although observation-space mismatches may appear benign in prior offline RL literature, we demonstrate that they pose a major obstacle for visual tasks. Future work will focus on exploring alternative implementations of the causal policy gradient, using different critic networks and more expressive policy classes.

\clearpage
\section*{Impact Statement}
This paper investigates the theoretical and algorithmic framework for robust policy learning from an expressive policy class using confounded offline data obtained by passively observing an expert demonstrator. Implicit confounding bias arises when the input variables that the expert uses to determine the action values are unknown and differ from the observed state used by the learner. \Cref{exp:_bias} illustrates a situation where the expert has access to the underlying structured states, while the learner is trying to learn from video recordings of the expert's natural trajectories. 

Our framework can be applied in various fields, including autonomous vehicle development, industrial robotics, and chronic disease management, among others. Especially for safety critical domains, our method mitigates the potential risks associated with offline reinforcement learning (RL) training from demonstrations that involve unobserved confounding.
We believe building trusted AI decision making systems are becoming increasingly crucial as the prevalence of black-box AI systems grows while our understanding of their long-term societal implications remains limited.
\bibliography{bibwithID,extra}
\bibliographystyle{icml2026}

\clearpage
\appendix
\onecolumn
\section{Related Work} \label{app:related}
\paragraph{Off-policy Causal Reinforcement Learning.}
When the no unobserved confounding assumption does not hold, one would need to either identify the reward and transition distributions before evaluating policy values or bound the possible policy values. 
There is a rich line of literature in identifying policy values directly from confounded data~\citep{DBLP:conf/icml/ShiUHJ22,DBLP:conf/nips/MiaoQZ22,DBLP:conf/icml/GuoCZYW22,DBLP:journals/ior/BennettK24}. But they usually invoke other critical learning assumptions such as the existence of bridge functions in the line of proximal causal inference literature \citep{tchetgen2020proximalcausalinference}. On the other hand, without further assumptions, one can utilize the bounding method to account for the whole range of possible policy values. Seminal work of Manski \citep{Manski1989NonparametricBO} developed the first bounds on causal effects in non-identifiable settings using observational data in the single-stage treatment model with contextual information (i.e., a contextual bandit model). These bounds were then expanded to the instrumental variable setting \citep{balke:pea97,imbens1994identification}, to partially identify counterfactual probabilities of causation \citep{DBLP:journals/amai/TianP00probcausationbound}, to construct reward shaping functions automatically \citep{li2025confoundedshaping}. This work is inspired by a recent work in partially identifying the policy values via bounding~\citep{zhang2025eligibility} and confounding robust off-policy learning in discrete action space~\citep{li2025confoundingdqn}.

\paragraph{Offline Reinforcement Learning.} Offline reinforcement learning~\citep{levine2020offline} concerns the problem of learning optimal policies from a static dataset collected by other behavioral policies while maintaining close affinity to the behavioral state-action distributions in the dataset. Several different strategies have been studied in the literature to enforce this constraint, including but not limited to conservatism~\citep{CQL2020Kumar,calql}, generative modeling~\citep{DBLP:conf/nips/ChenLRLGLASM21,DBLP:conf/nips/JannerLL21}, OOD detection~\citep{garg2023extremeqlearningmaxentrl}, and behavioral regularization~\citep{fujimoto2021a,tarasov2023rebrac,park2025fql,dong2025valueflows,agrawalla2025floq}. After offline training, an online finetuning phase can be added to further improve the agent's performance. To avoid sudden distribution shift cased unlearning~\citep{calql}, one could use balanced sampling~\citep{hybridrl,10.5555/3666122.3666329hybridpolicyopt}, maintaining offline data in the reply buffer~\citep{park2025fql,nair2021awacacceleratingonlinereinforcement,efficientoffon} or penalize Q-values~\citep{calql}.

\paragraph{Flow Matching in RL.} Due to the powerful multi-modal distribution modeling capability of flow matching and relatively easy implementation compared against solving PDEs as in diffusion based methods, flow matching~\citep{lipman2022flow,liu2022rectifiedflow,tong2024conditionalflow} has gained great attention in the RL community recently. Its applications in reinforcement learning spans from planning~\citep{nandiraju2025hdflow,nguyen2025flowmplearningmotionfields}, learning world models~\citep{liu2025towardsfoundationallidar,rohbeck2025modelingcomplexflow,lillemark2026flowequivariantworldmodels}, online learning~\citep{lv2025flowbasedonlinerl}, offline learning~\citep{park2025fql,dong2025valueflows,agrawalla2025floq,tiofack2025guidedflowpolicylearning,zhang2026sacflowsampleefficientreinforcement,ghugare2025normalizingflowscapablemodels} to multi-modal vision-language-action models (VLAs)~\citep{black2026pi0visionlanguageactionflowmodelpi0,intelligence2025pi05visionlanguageactionmodelopenworld,deng2025graspvla,jiang2025AsyncVLA}.

\section{Limitations} \label{app:limitation}
There are mainly two types of limitations of the proposed Causal-FQL. The first type of limitations pertain to FQL~\citep{park2025fql} that 1) it requires numerically solving ODEs during training resulting in potentially slow training; 2) FQL only deploys the simplest form of flow matching. Combining with recent advancements like rectified flows~\citep{liu2022rectifiedflow} and optimal transport conditional flow matching~\citep{tong2024conditionalflow} should bring faster convergence, stabler training and better performance; 3) FQL does not have a built in exploration mechanism, which could limit its online fine-tuning performance. The second type of limitations are from approximating the causal offline RL objective. As we introduce more components to the pipeline of FQL, the combinatorial space of tunable hyper-parameters is even harder to tune than before. Given the computational heavy nature of pixel-based offline RL tasks, a simpler confounding robust offline RL method is our next step before larger scale applications.

\section{Extended Preliminaries and Proof} \label{app:proof}
\paragraph{Causal Foundations.}
We briefly recap on the foundations of causal inference to facilitate our discussion. Our definition of CMDP is closely tied to the definition of Structural Causal Models \citep{pearl2009causality}. 
\begin{definition}
An SCM $\1M$ is a tuple $\tuple{\*U, \*V, \2F, P}$ where
\begin{enumerate}[label=\textbullet, leftmargin=23pt, topsep=0pt, parsep=0pt, itemsep=1pt]
    \item $\*U$ is a set of exogenous variables;
    \item $\*V$ is a set of endogenous variables;
    \item $\2F$ is a set of functions s.t. each $f_V \in \2F$;
    \item the exogenous distribution $P(\*U)$. 
\end{enumerate}
$f_V$ decides values of an endogenous variable $V \in \*V$ taking as argument a combination of other variables in the system. That is, $V \leftarrow f_{V}(\*\PA_V, \*U_V), \*\PA_V \subseteq \*V, \*U_V \subseteq \*U$. Values of exogenous variables $\*U$ are drawn from $P(\*U)$.
\end{definition}
A policy $\pi$ over a subset of variables $\*X \subseteq \*V$ is a sequence of decision rules $\left\{ \pi(X | \*S_X) \right\}_{X \in \*X}$, where every $\pi(X |\*S_X)$ is a probability distribution mapping from domains of a set of covariates $\*S_X \subseteq \*V$ to the domain of action $X$. An intervention following a policy $\pi$ over variables $\*X$, denoted by $\doo(\pi)$, is an operation which sets values of every $X \in \*X$ to be decided by policy $X \sim \pi(X | \*S_X)$ \citep{sigmacalculus}, replacing the functions $f_{\*X} = \{f_{X}: \forall X \in \*X\}$ that would normally determine their values. For an SCM $\1M$, let $\1M_{\pi}$ be a submodel of $M$ induced by intervention $\doo(\pi)$. For a set $\*Y \subseteq \*V$, the interventional distribution $\inv{\*Y}{\pi}$ is defined as the distribution over $\*Y$ in the submodel $\1M_{\pi}$, i.e., $\Inv{\*Y}{\pi}{\1M}\triangleq P_{\1M_{\pi}}\Parens{\*Y}$; restriction $\M$ is left implicit when it is obvious.

Each SCM $\1M$ is also associated with a causal diagram $\G$ (e.g., \Cref{fig:_2_1_mdp}), which is a directed acyclic graph (DAG) where nodes represent endogenous variables $\*V$ and arrows represent the arguments $\*\PA_V, \*U_V$ of each structural function $f_{V} \in \2F$. Exogenous variables $\*U$ are often not explicitly shown by convention. However, a bi-directed arrow $V_i \leftrightarrow V_j$ indicates the presence of an unobserved confounder (UC), $U_{i, j} \in \*U$ affecting $V_i, V_j$, simultaneously~\citep{PCH}. We will use standard graph-theoretic family abbreviations to represent graphical relationships, such as parents ($\pa$), children ($\ch$), descendants ($\de$), and ancestors ($\an$). For example, the set of parent nodes of $\*X$ in $\1G$ is denoted by $\pa(\*X)_{\1G} = \cup_{X \in \*X} \pa(X)_{\1G}$. Capitalized versions $\Pa, \Ch, \De, \An$ include the argument as well, e.g., $\Pa(\*X)_{\1G} = \pa(\*X)_{\1G} \cup \*X$. 
A path from a node $X$ to a node $Y$ in $\G$ is a sequence of edges that does not include a particular node more than once. Two sets of nodes $\*X, \*Y$ are said to be d-separated by a third set $\*Z$ in a DAG $\G$, denoted by $(\*X \ci \*Y | \*Z)_{\G}$, if every edge path from nodes in $\*X$ to nodes in $\*Y$ is ``blocked'' by nodes in $\*Z$. The criterion of blockage follows \citet[Def.~1.2.3]{pearl2009causality}. For more details on SCMs, we refer readers to \citet{pearl2009causality,PCH}.

\paragraph{From MDP to CMDP.} In MDP, all the state variables are observable to the agent. When transitioning from the demonstrator's perspective to the learner's perspective, this assumption is implicitly assumed apriori. However, as we have demonstrated in the main text, this may not be usually the case. And we argue that the same problem also lurks in other aspects of reinforcement learning~\citep{crlsurvey}, like and off-policy learning~\citep{10.1609safepartialid,zhang2019near,li2025confoundingdqn}, curriculum learning~\citep{curriculum_survey,li2024causally}, and reward shaping~\citep{li2025confoundedshaping}. This motivates our proposal of CMDP as a proper modeling of the confounding biases while maintaining the Markov property for efficient policy learning. 

Compared to MDP graphically, CMDP only adds three bi-directed arrows within each time step. From the definition of causal diagram we discussed above, this indicates that the three variables, current time step action $X_t$, reward $Y_t$ and next state $S_{t+1}$, may be affected by the same latent confounding variables. In the pixel-based offline RL tasks, those variables could be occluded joint angles, exact angular velocities of the end effector or the precise location of an object, all kind of information that the demonstrator's policy relies on but not fully observable to the learner via pixel observations.

\paragraph{Proof of Theorem 3.1}
We are now ready to provide the detailed proof for the causal policy gradient objective in \Cref{thm:_3_1}.
\begin{theorem}[Restatement of \cref{thm:_3_1}]
	For a CMDP environment $\M$ with reward signals $Y_t \in [a, b] \subseteq \3R$, fix a policy $\pi$. The state value function $V_{\pi}(s) \geq \underline{V_{\pi}}(s)$ for any state $s \in \1S$, where the lower bound $\underline{V_{\pi}}(s)$ is given by as follows,
	\begin{align}
		&\underline{V_{\pi}}(s) = \3E_{\substack{x \sim p(\cdot \mid s)\\ x', x^*\sim \pi(\cdot \mid s)}} \bigg[\I_{x \neq x'} \Big (a + \gamma \min_{s^*} \underline{Q_{\pi}}(s^*, x^*) \Big) + \I_{x = x'} \Big (\widetilde{\1R}\Parens{s, x} + \gamma \sum_{s'} \widetilde{\1T}\Parens{s, x, s'}\underline{Q_{\pi}}(s', x^*)\Big) \bigg]
	\end{align}
\end{theorem}
\begin{proof}
Since $\underline{V_{\pi}}(s) = \sum_{x} \pi(x \mid s) \underline{Q_{\pi}}(s, x)$, it follows from \Cref{eq:_3_lower_q_2} that the state value function is lower bounded by
\begin{align}
    \underline{V_{\pi}}(s) &= \sum_{x} \pi(x \mid s) \mu(\neg x \mid s) \bigg (a + \gamma \min_{s'} \underline{V_{\pi}}(s') \bigg ) \\
    &+  \sum_{x} \pi(x \mid s)\mu(x\mid s)\bigg (\widetilde{\1R}\Parens{s, x} + \gamma \sum_{s', x'} \widetilde{\1T}\Parens{s, x, s'} \underline{V_{\pi}}(s') \bigg )
\end{align}
Following the simplification step used in \citep[Theorem 1]{zhang2024eligibility}, 
\begin{align}
    \underline{V_{\pi}}(s) &= \sum_{x} \mu(x \mid s) \pi(\neg x \mid s)\bigg (a + \gamma \min_{s'} \underline{V_{\pi}}(s') \bigg ) \\
    &+  \sum_{x} \mu(x\mid s)\pi(x \mid s)\bigg (\widetilde{\1R}\Parens{s, x} + \gamma \sum_{s', x'} \widetilde{\1T}\Parens{s, x, s'} \underline{V_{\pi}}(s') \bigg )
\end{align}
The above equation could be further written as
\begin{align}
    \underline{V_{\pi}}(s) &= \sum_{x, x'} \mu(x \mid s) \pi(x' \mid s) \I_{x \neq x'} \bigg (a + \gamma \min_{s'} \underline{V_{\pi}}(s') \bigg ) \\
    &+  \sum_{x, x'} \mu(x\mid s)\pi(x' \mid s) \I_{x = x'} \bigg (\widetilde{\1R}\Parens{s, x} + \gamma \sum_{s', x'} \widetilde{\1T}\Parens{s, x, s'} \underline{V_{\pi}}(s') \bigg )
\end{align}
Finally, replacing the summation over the policy product $\mu(x\mid s)\pi(x' \mid s)$ with the expectation proves the statement.
\end{proof}

\section{Additional Results} \label{app:extra exp}
Other than manipulation tasks, we also test Causal-FQL in visual locomotion tasks like \texttt{visual-antmaze} series. The results are not as amazing as those on manipulation tasks, which are as expected. We hypothesize that this is because FQL and Causal-FQL sample actions from the policy directly without considering the noisy distribution of Q-values while Gaussian policy extraction with rejection sampling handles uni-modal reward distributions better. Thus, prior Gaussian SOTA like BRAC performs well in locomotion tasks. As an indirect proof to our hypothesis, in \cref{tab:locomotion offline}, ReBRAC \citep{tarasov2023rebrac} also outperforms its own flow-based counterpart, FBRAC \citep{park2025fql}, and IFQL (FQL with rejection sampling) also outperforms FQL.

\begin{table*}[ht]
\centering
\label{tab:locomotion offline}
\scriptsize 
\setlength{\tabcolsep}{3pt} 
\renewcommand{\arraystretch}{1.1} 
\newcommand{\tb}[1]{\mathbf{#1}} 
\caption{\textbf{Locomotion offline evaluation results.} Both \texttt{Causal-FQL} and vanilla FQL are not on par with Gaussian based methods. We report results averaged over 4 seeds and bold values within 95\% of the best performance of each domain following prior work \citep{park2025fql}.}
\resizebox{\textwidth}{!}{%
\begin{tabular}{l cc ccccc}
\toprule
& \multicolumn{2}{c}{Gaussian Policies} & \multicolumn{5}{c}{Flow Policies} \\
\cmidrule(lr){2-3} \cmidrule(lr){4-8}
& IQL & ReBRAC & FBRAC & IQN & IFQL & FQL & \texttt{Causal-FQL} \\
\midrule

\ttfamily visual-antmaze-medium-navigate-singletask-task1-v0 (*) & $\tb{78 \pm 9}$ & $54 \pm 15$ & $27 \pm 3$ & $62 \pm 7$ & $\tb{81 \pm 3}$ & $32 \pm 3$ & ${40 \pm 15}$ \\
\ttfamily visual-antmaze-medium-navigate-singletask-task2-v0 & $90 \pm 3$ & $\tb{96 \pm 1}$ & $42 \pm 4$ & $88 \pm 2$ & $87 \pm 1$ & $60 \pm 2$ & ${61 \pm 2}$\\
\ttfamily visual-antmaze-medium-navigate-singletask-task3-v0 & $80 \pm 6$ & $\tb{97 \pm 1}$ & $32 \pm 4$ & $64 \pm 5$ & $92 \pm 1$ & $35 \pm 8$ & ${45 \pm 3}$ \\
\ttfamily visual-antmaze-medium-navigate-singletask-task4-v0 & $\tb{89 \pm 4}$ & $\tb{93 \pm 2}$ & $23 \pm 2$ & $71 \pm 6$ & $84 \pm 3$ & $35 \pm 2$ & ${39 \pm 2}$\\
\ttfamily visual-antmaze-medium-navigate-singletask-task5-v0 & $84 \pm 2$ & $\tb{97 \pm 1}$ & $25 \pm 4$ & $84 \pm 1$ & $89 \pm 2$ & $29 \pm 7$ & ${29 \pm 6}$\\
\midrule

\ttfamily visual-antmaze-teleport-navigate-singletask-task1-v0 (*) & $5 \pm 2$ & $2 \pm 0$ & $1 \pm 1$ & $2 \pm 1$ & $\tb{7 \pm 4}$ & $2 \pm 1$ & ${4 \pm 2}$ \\
\ttfamily visual-antmaze-teleport-navigate-singletask-task2-v0 & $10 \pm 2$ & $10 \pm 3$ & $6 \pm 5$ & $7 \pm 3$ & $\tb{13 \pm 3}$ & $6 \pm 1$ & ${7 \pm 4}$\\
\ttfamily visual-antmaze-teleport-navigate-singletask-task3-v0 & $7 \pm 7$ & $4 \pm 1$ & $10 \pm 4$ & $6 \pm 4$ & $8 \pm 9$ & $9 \pm 4$ & $\tb{14 \pm 5}$\\
\ttfamily visual-antmaze-teleport-navigate-singletask-task4-v0 & $4 \pm 6$ & $4 \pm 0$ & $10 \pm 2$ & $4 \pm 2$ & $\tb{18 \pm 2}$ & $9 \pm 1$ & ${11 \pm 3}$\\
\ttfamily visual-antmaze-teleport-navigate-singletask-task5-v0 & $2 \pm 1$ & $2 \pm 1$ & $2 \pm 1$ & $2 \pm 1$ & $\tb{4 \pm 2}$ & $1 \pm 1$ & $\tb{4 \pm 2}$\\
\bottomrule
\end{tabular}
}
\end{table*}

\section{Implementation Details} \label{app:implement detail}
\texttt{Causal-FQL} is built upon the codebase of FQL~\citep{park2025fql}\footnote{\textcolor{betterblue}{https://github.com/seohongpark/fql}}. \texttt{Causal-FQL} inherits most of the architectural design of FQL with minimum changes implemented for the causal offline RL objective. First, we add an action discriminator and a confounding‑robust Q weighting scheme: it builds an extra action discriminator network (the same size as the Q-critic), uses it in actor loss to classify flow actions vs one-step target flow actions, computes factual weights from discriminator logits, and mixes ensemble Q-values as \cref{eq:approx causal q obj}; it also adds a discriminator loss term with optional exponential decay and a switch during online learning phase that disables discriminator weighting and trains with plain mean Q values over the ensembles. Due to limited computation resources, we don't have the capacity to fully tune the full hyperparameter space of \texttt{Causal-FQL} on visual tasks. But we do observe consistent performance improvement in all visual tasks over FQL using the default FQL hyperparameters.

\section{Experiment Details} \label{app:exp detail}
\paragraph{Tasks and Environments.} We use pixel based tasks from OGBench \citep{park2025ogbench} for evaluation. We use a total of 25 visual tasks across 5 environments for evaluation. Below is a list of the datasets we used, each of which defines 5 reward maximizing tasks indexed from \texttt{-task1} to \texttt{-task5}. In the main text, we evaluated manipulation tasks involving multi-step of subtasks with rewards bounded by the negative of number of subtasks and 0. In \cref{app:extra exp}, we evaluate locomotion tasks with rewards -1 or 0 indicating whether the goal is reached. 
\begin{enumerate}[label=\textbullet, leftmargin=23pt, topsep=0pt, parsep=0pt, itemsep=5pt]
    \item \texttt{visual-cube-single-play-singletask} %
    \item \texttt{visual-cube-double-play-singletask} %
    \item \texttt{visual-scene-play-singletask} %
    \item \texttt{visual-puzzle-3x3-play-singletask} %
    \item \texttt{visual-puzzle-4x4-play-singletask} 
\end{enumerate}

\paragraph{Methods and Hyperparameters.} We adopt the baseline implementation from FQL~\citep{park2025fql} and Value Flows~\citep{dong2025valueflows}.\footnote{\textcolor{betterblue}{https://github.com/chongyi-zheng/value-flows}}
Below we list the baseline methods we used. We take default hyper-parameters recommended in FQL and Value Flows code base. See the attached link in footnote for detailed command to replicate the results.
\begin{enumerate}[label=\textbullet, leftmargin=23pt, topsep=0pt, parsep=0pt, itemsep=5pt]
    \item \textbf{IQL} \citep{kostrikov2021iql}. Implicit Q-Learning use expectile regression to represent Q-values which is then used to select optimal actions. We use the default inverse temperature $\alpha$ recommended in the codebase.
    \item \textbf{ReBRAC} \citep{tarasov2023rebrac}. ReBRAC is a Gaussian policy based offline RL method that uses TD3 to learn values with behavioral regularizations. 
    \item \textbf{FBRAC} \citep{park2025fql}. A variant of BRAC with flow policies instead. It doesn't use the one-step flow policy extraction thus requiring backpropagation through time.
    \item \textbf{IQN} \citep{pmlr-v80-dabney18aiqn}. IQN is a distributional RL baseline that represents the Q-value distribution with quantile values. Action selection is based on taking argmax over sampled actions' Q-values.
    \item \textbf{IFQL} \citep{park2025fql}. A variant of Implicit Diffusion Q (IDQL, \cite{idql}) with flow poilcy and rejection sampling action selection, which is first proposed in \citep{park2025fql}.
    \item \textbf{FQL} \citep{park2025fql}. FQL uses one-step flow policy to maximize the Q-values learned by an ensemble of critics. It uses a BC flow policy as the behavioral regularization. 
\end{enumerate}

\section{Causal Value Flows} \label{app:value flow}

\begin{wrapfigure}{r}{0.33\textwidth}
    \vspace{-0.1in}
    \centering
    \includegraphics[width=\linewidth]{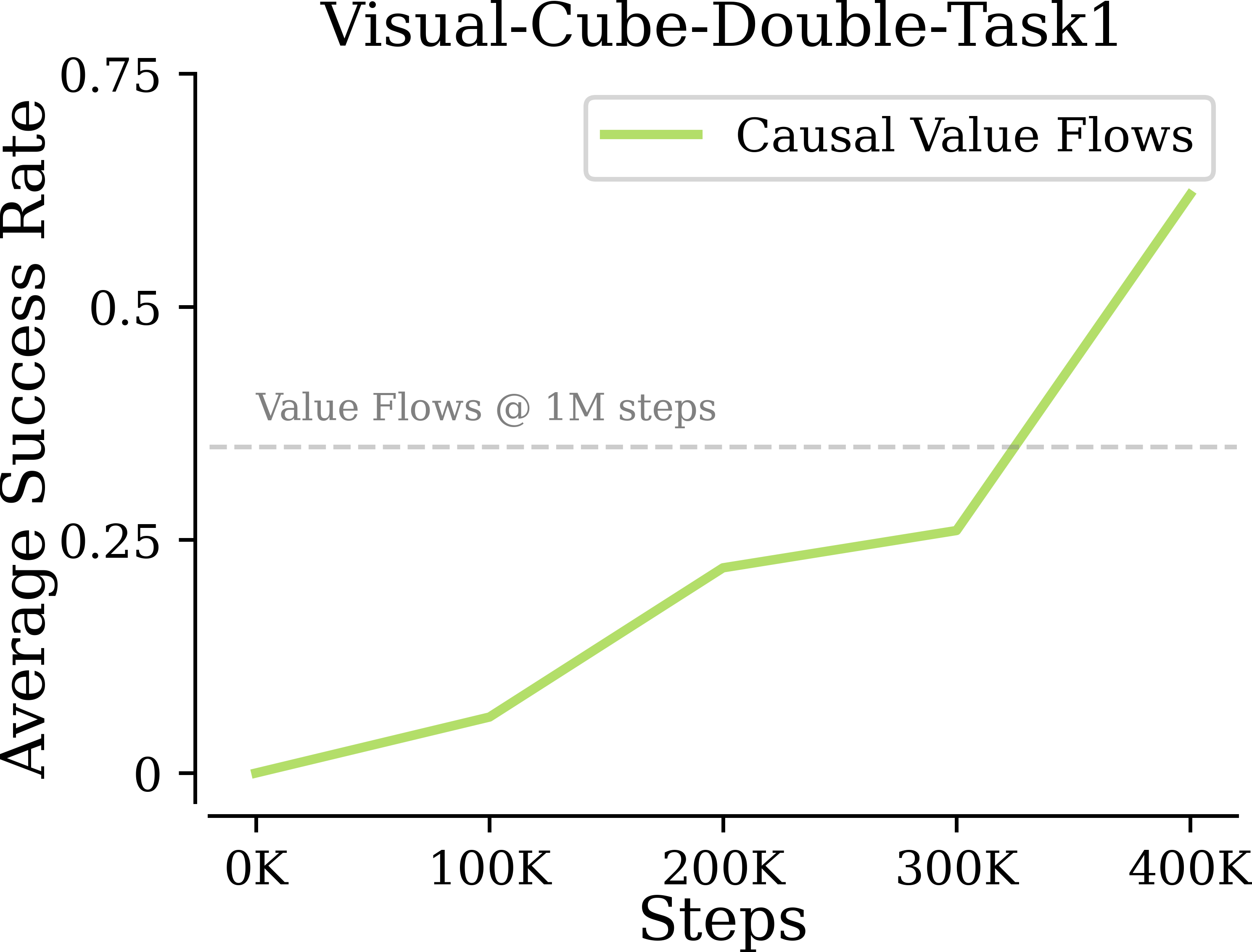}
    \caption{\textbf{Causal Value Flows outperform vanilla Value Flows.}}
    \label{fig:placeholder}
\end{wrapfigure}
Here we describe the implementation details of Causal Value Flows and preliminary offline evaluation results for applying the newly proposed causal offline RL objective to Value Flows~\citep{dong2025valueflows} as a demonstration on the versatility of our method. We adopt the discriminator coefficient and number of ensembles setup from Causal-FQL and hold other parameters as default in value flows. As in Causal-FQL, we add an extra action discriminator network, the same size as the Q-critic network. When calculating \cref{eq:approx causal q obj}, we use the values sampled by Euler method with the trained return vector field ensembles. Then we reweight the mean Q-values and add the minimum Q-value as the worst case value. Other parts of value flows stay unchanged.

Without further tuning, on task 1 of \texttt{visual-cube-double-play}, the augmented causal value flow is able to achieve over 0.6 success rate within 400K offline training steps, \textbf{almost doubles the vanilla value flow's performance}, which only achieves 0.32 with 1M offline training steps.

\end{document}